\documentclass{article}

\usepackage{microtype}
\usepackage{graphicx}
\usepackage{booktabs} 

\usepackage[accepted]{icml2021}

\usepackage{hyperref}

\usepackage{url}
\usepackage{xcolor}
\usepackage{paralist}
\usepackage{subcaption}
\usepackage{mathtools}
\usepackage{amsthm}

\graphicspath{ {./images/} }

\usepackage{amsmath,amsfonts,bm}

\newcommand{\clustrad}{r}

\def\figref#1{figure~\ref{#1}}

\def\eqref#1{equation~\ref{#1}}

\def\1{\bm{1}}

\def\ry{{\textnormal{y}}}

\def\rvx{{\mathbf{x}}}

\def\vc{{\bm{c}}}

\def\vh{{\bm{h}}}

\def\vu{{\bm{u}}}
\def\vv{{\bm{v}}}
\def\vw{{\bm{w}}}
\def\vx{{\bm{x}}}
\def\vy{{\bm{y}}}

\def\mW{{\bm{W}}}

\DeclareMathAlphabet{\mathsfit}{\encodingdefault}{\sfdefault}{m}{sl}
\SetMathAlphabet{\mathsfit}{bold}{\encodingdefault}{\sfdefault}{bx}{n}

\def\gN{{\mathcal{N}}}

\def\sI{{\mathbb{I}}}

\def\sN{{\mathbb{N}}}

\def\sP{{\mathbb{P}}}

\def\sS{{\mathbb{S}}}

\def\sU{{\mathbb{U}}}
\def\sV{{\mathbb{V}}}
\def\sW{{\mathbb{W}}}
\def\sX{{\mathbb{X}}}
\def\sY{{\mathbb{Y}}}

\newcommand{\R}{\mathbb{R}}

\DeclareMathOperator*{\argmin}{arg\,min}

\newcommand{\wvec}[1]{\vw^{(#1)}}
\newcommand{\uvec}[1]{\vu^{(#1)}}

\newcommand{\whvec}[1]{\hat{\vw}^{(#1)}}
\newcommand{\uhvec}[1]{\hat{\vu}^{(#1)}}
\newcommand{\vhvec}{\hat{\vv}}

\newcommand{\rpm}{\raisebox{.2ex}{$\scriptstyle\pm$}}
\newcommand\alonb[1]{\textcolor{blue}{[AB: #1]}}
\newcommand\amirg[1]{\textcolor{purple}{[AG: #1]}}

\newcommand\ignore[1]{}
\renewcommand{\eqref}[1]{Eq. (\ref{#1})}
\newcommand{\defref}[1]{Def. (\ref{#1})}
\newcommand{\thmref}[1]{Theorem. (\ref{#1})}
\newcommand{\corref}[1]{Corollary (\ref{#1})}
\newcommand{\lemref}[1]{Lemma. (\ref{#1})}

\renewcommand{\figref}[1]{Figure \ref{#1}}
\newcommand{\sgn}[1]{\mbox{sign}\left({#1}\right)}

\newtheorem{thm}{Theorem}[section]
\newtheorem{defn}{Definition}[section]

\newtheorem{lem}{Lemma}[section]
\newtheorem{prop}{Proposition}[section]
\newtheorem{cor}{Corollary}[section]
\newtheorem{assumption}{Assumption}[section]

\icmltitlerunning{Towards Understanding Learning in Neural Networks with Linear Teachers}

\begin{document}




\twocolumn[
\icmltitle{Towards Understanding Learning in Neural Networks with Linear Teachers}




\begin{icmlauthorlist}
\icmlauthor{Roei Sarussi}{tau}
\icmlauthor{Alon Brutzkus}{tau}
\icmlauthor{Amir Globerson}{tau}
\end{icmlauthorlist}

\icmlaffiliation{tau}{The Blavatnik School of Computer Science, Tel Aviv University}

\icmlcorrespondingauthor{Alon Brutzkus}{alonbrutzkus@mail.tau.ac.il}

\icmlkeywords{Machine Learning, ICML}

\vskip 0.3in
]

\printAffiliationsAndNotice{}




\begin{abstract}
    Can a neural network minimizing cross-entropy learn linearly separable data? Despite progress in the theory of deep learning, this question remains unsolved.
    Here we prove that SGD globally optimizes this learning problem for a two-layer network with Leaky ReLU activations. The learned network can in principle be very complex. However, empirical evidence suggests that it often turns out to be approximately linear. We provide theoretical support for this phenomenon by 
    proving that if network weights converge to two weight clusters, this will imply an approximately linear decision boundary.
    Finally, we show a condition on the optimization that leads to weight clustering. We provide empirical results that validate our theoretical analysis.
\end{abstract}

\section{Introduction}

Neural networks have achieved remarkable performance in many machine learning tasks \citep{krizhevsky2012imagenet, silver2016mastering, devlin2019bert}. Although their success has already transformed technology, a theoretical understanding of how this performance is achieved is not complete. Here we focus on one of the simplest learning settings that is still not understood. We consider linearly separable data (i.e., generated by a ``linear teacher'') that is being learned by a two layer neural net with leaky ReLU activations and minimization of cross entropy loss using gradient descent or its variants. Two key questions immediately come up in this context:
\begin{itemize}
    \item {\bf The Optimization Question}: Will the optimization succeed in finding a classifier with zero training error, and arbitrarily low training loss?
    \item {\bf The  Inductive Bias Question}: With a large number of hidden units, the network can find many solutions that will separate the data. Which of these will be found by gradient descent? 
\end{itemize}
Our work addresses these questions as follows.
\newline
{\bf The Optimization Question}: We prove that stochastic gradient descent (SGD) will converge to arbitrary low training loss. Concretely, we show that for any $\epsilon>0$, SGD will converge to $\epsilon$ cross-entropy loss in $O\left(\frac{1}{\epsilon^2}\right)$ iterations. We consider SGD which performs multiple passes over the data and we devise a novel variant of the perceptron proof to analyze this setting. Our analysis bounds the number of epochs that have high loss examples, and uses this to show convergence to a low loss solution.
Importantly, our result holds for any network size and scale of initialization. Therefore, our analysis goes beyond the Neural Tangent Kernel (NTK) analyses which require large network sizes and relatively large initialization scales.

 {\bf The Inductive Bias Question}: We empirically observe that when a small initialization scale is used, the learned network converges to a decision boundary that is very close to linear. See \figref{fig:intro small initialization landscape} for a 2D example.\footnote{See Section \ref{sec:cluster_empirical} and Section \ref{sec:empirical_nar} for more empirical examples.} We also observe that all neurons cluster nicely into two sets of vectors (i.e., they form two groups of well-aligned neurons) as in \figref{fig:intro small initialization neurons}. To support these empirical findings, we provide the following theoretical results: \\ 
 \textbf{(1)} We prove that an approximate clustering of the neurons implies that the decision boundary of the network is approximately linear. This is a result of a nice property of leaky ReLU networks which we prove in Section \ref{sec:quasi_linear}. \\
 \textbf{(2)} We provide a novel sufficient condition on the optimization path of gradient flow which implies convergence to clustered solutions. With the result above, it implies convergence to a linear decision boundary. The condition states that from a certain iteration on, all neurons with the same output sign ``agree'' on the classification of the data. We observe that this condition holds empirically for several synthetic and real datasets. Finally, we use the latter result to prove that under certain assumptions, the learned network is a solution to an SVM problem with a specific kernel.

    
    \comment{
    a large range of initialization scales, \alonb{This is confusing, because the reader may think that this also includes large initialization scales.} the network learns a decision boundary that is very close to {\em linear} (note that the network can model highly non-linear rules that separate the data (see Figure \ref{fig:intro fig}). Additionally, we observe that all neurons cluster nicely into two sets of vectors (i.e., they form two groups of well-aligned neurons).\amirg{add figure for this} \alonb{Say that even in certain cases (as in the figure) we see that the decision boundary is exactly linear.}
    
    \alonb{I think it would be good to add a figure where the decision boundary is approximately linear. E.g., in the case of a Gaussian distribution with outliers.}
    }


Our results above make significant headway in understanding why optimization is tractable with linear teachers, and why convergence is to approximately linear boundaries. We also provide empirical evaluation that confirms that weight clustering indeed explains why approximate linear decision boundaries are learned. 
\begin{figure*}[tb]
    \begin{subfigure}[b]{0.475\textwidth}
        \centering
        \includegraphics[width=0.5\textwidth,height=0.5\textwidth]{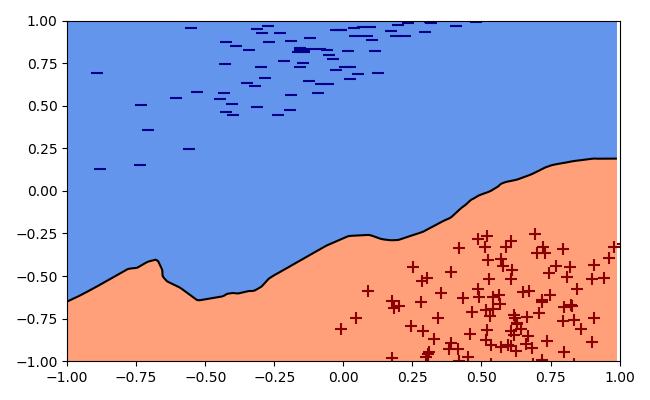}
        \caption[Network2]%
        {{\small Decision boundary (large initialization)}}    
        \label{fig:intro large initialization landscape}
    \end{subfigure}
    \hfill
    \begin{subfigure}[b]{0.475\textwidth}  
        \centering 
        \includegraphics[width=0.5\textwidth,height=0.5\textwidth]{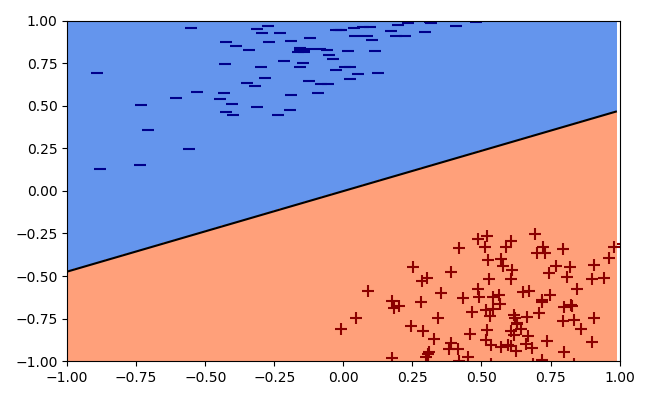}
        \caption[]%
        {{\small Decision boundary (small initialization)}}    
        \label{fig:intro small initialization landscape}
    \end{subfigure}
    \vskip\baselineskip
    \begin{subfigure}[b]{0.475\textwidth}   
            \centering 
            \includegraphics[width=0.5\textwidth,height=0.5\textwidth]{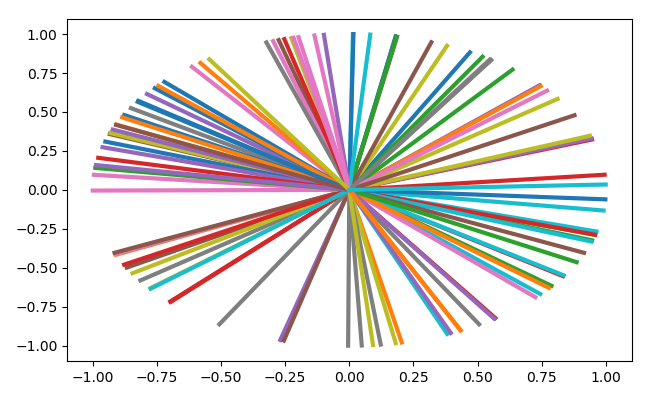}
            \caption[]%
            {{\small Learned Neurons (large initialization)}}    
            \label{fig:intro large initialization neurons}
        \end{subfigure}
        \hfill
        \begin{subfigure}[b]{0.475\textwidth}   
            \centering 
            \includegraphics[width=0.5\textwidth,height=0.5\textwidth]{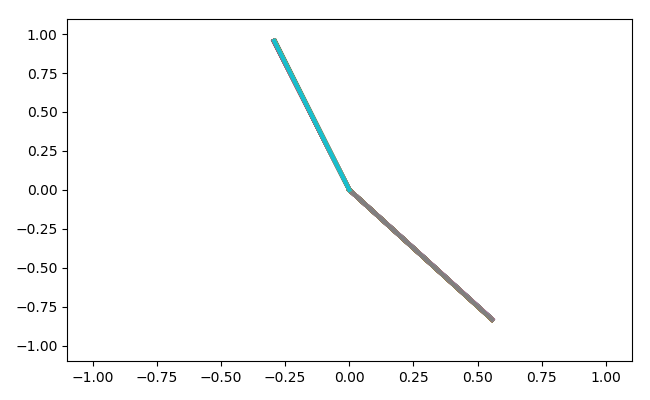}
            \caption[]%
            {{\small Learned Neurons (small initialization)}}    
            \label{fig:intro small initialization neurons}
        \end{subfigure}
    \caption[ ]
    {\small{Results for training a Leaky-ReLU network on linearly separable data, with different initialization scales. 
    Figures (a)+(b) show the resulting decision boundary. It can be seen that large initialization leads to a non-linear boundary, whereas small initialization leads to a linear boundary. Figures (c)+(d) show the learned weight vectors (normalized to unit norm). It can be seen that small initialization leads to two tight clusters of neurons whereas large initialization does not lead to clustering. The network has $100$ neurons, initialized from a Gaussian with standard deviation $0.001$ for small initialization and $30$ for large initialization.}}
    \label{fig:intro fig}
\end{figure*}


\comment{
In this work, we provide an in depth examination of the implicit bias gradient methods yield across all the domains of the optimization, in the case of optimizing a 1 hidden layer  Leaky ReLU network with logistic loss. More specifically, we prove that:
}

\section{Related Work}


Since training neural networks is NP-Hard for worst-case datasets \citep{blum1992training}, recent works have analyzed neural networks under certain data assumptions to better understand their performance in practice. One common assumption is to analyze neural networks when the data is linearly separable. Even in this case, the theoretical analysis of optimization and generalization is far from  resolved. In a work closely related to ours, \citet{brutzkus2018sgd} consider this setting and show that SGD converges to zero loss for linearly separable data (which was later extended to ReLU activations using noisy SGD in \citet{Wang_2019}). The key difference from our work is that they use the hinge loss instead of the cross entropy loss. The cross entropy loss creates unique challenges for proving convergence as we show in Section \ref{sec:risk_convergence}. Thus, their results cannot be directly applied for the cross entropy loss and we use novel techniques to guarantee convergence of SGD in this case. The second key difference from their work is that we present novel insights on the inductive bias of SGD using results of \citet{lyu2019gradient} and \citet{ji2020directional} which hold for the cross entropy loss and not for the hinge loss. Finally, our result which shows that a network with clustered neurons has an approximate linear decision boundary is new and holds irrespective of the loss used. 

Recently, \citet{lampert21_orthogonal} analyzed a subclass of linear teachers where data is ``orthogonally separable''. In this case they show that training a ReLU network with the cross entropy loss results in a solution where weights are aligned. In terms of our results, this can be viewed as a case of convergence to a particular PAR (see Section \ref{sec:exact_linear}).  Several other works assume that the data is linearly separable but also that the networks are \textit{linear}  \citep{ji2018gradient, moroshko2020implicit,gunasekar2018implicit}. We study the more challenging and realistic setting of two layer nonlinear networks with Leaky ReLU activations.

\citet{}

Several works \cite{lyu2019gradient, ji2020directional, nacson2019lexicographic} studied the inductive bias of two-layer homogeneous networks and showed connections between gradient methods and margin maximization. Their results hold under the assumption that gradient methods achieve a certain loss value. However, we provide a convergence proof for SGD that shows that it can obtain arbitrary low loss values. Furthermore, we use the results of \citet{lyu2019gradient} and \citet{ji2020directional} to obtain a more fine grained analysis of the inductive bias of gradient flow for linear teachers. Other works considered the inductive bias of infinite two-layer networks \citep{ChizatB20, chizat2018global, wei2019regularization, mei2018mean}. Our results hold for networks of any size. An inductive bias towards clustered solutions has been observed in \citet{brutzkus2019larger} and proved for a simple setup with nonlinear data.

Fully connected networks were also analyzed via the NTK approximation \citep{du2018gradient2, du2018gradient, arora2019fine, ji2019polylogarithmic, cao2019generalization, jacot2018neural, fiat2019decoupling, allen2019convergence, li2018learning, daniely2016toward}. 
However other  works \citep{yehudai2019power, daniely2020learning} have highlighted limitations of the NTK framework, suggesting that it does not accurately model neural networks as they are used in practice. Our convergence analysis in Section \ref{sec:risk_convergence} holds for any initialization scale and network size and therefore goes beyond the NTK analysis.

Recently, \citet{li2020learning} analyzed two-layer networks beyond NTK in the case of Gaussian inputs and squared loss. We assume linearly separable inputs and the cross entropy loss.  \citet{allen2019can} analyze a three layer ResNet and provide generalization guarantees for sufficiently wide networks in a regression setting. \citet{WoodworthGLMSGS20} study the inductive bias of gradient methods for a simplified nonlinear model. 


\section{Preliminaries}
\label{sec:formulation}

{\bf Notations:} We use $|| \cdot ||$ to denote the $L^{2}$ norm on vectors and Frobenius norm on matrices. For a vector $\vv$ we denote $\vhvec = \frac{\vv}{\left\|\vv\right\|}$.

{\bf Data Generating Distribution:} Define $\sX = \{\vx \in \R^{d} :  || \vx || \leq R_{x} \}$ and $\sY = \{ \rpm 1\}$.
We consider a  distribution of linearly separable points. Formally, let $D(\rvx,\ry)$ be a distribution over $\sX \times \sY$ such that there exists $\vw^{*} \in \R^{d}$ for which $\displaystyle \sP_{(\rvx,\ry) \sim D}[\ry \vw^{*} \cdot \vx \geq 1]=1$.
Let $\sS \coloneqq \{(\vx_1,y_1),....,(\vx_n,y_n)\} \subseteq \sX \times \sY$ be a training set sampled IID from $D(\rvx,\ry)$. Let $\sS_{+}\subset \sS$ denote the training points with positive labels and $\sS_{-}$ with negative labels. We denote $\vx \in \sS$  if there exists $y \in \{\pm 1\}$ such that $(\vx,y) \in \sS$.


{\bf Network Architecture:}
We consider a two-layer neural network with $2k >0 $ hidden units, where the second layer is fixed and the first layer is learned. Formally, we denote the parameters of the network that are learned by $\mW \in \R^{2k \times d}$ and for the second layer we define the \textit{fixed} vector $\vv \in \R^{2k}$, where $\displaystyle \vv = (\overbrace{v,v,...,v}^{k},\overbrace{-v,...,-v}^{k})$ and $v>0$. The network output is given by the function $N_{\mW} : \R^d \rightarrow \R$ defined as $N_{\mW}(\vx) = \vv \cdot \sigma\left(\mW\vx\right)$, where $\sigma(\vx) = \max\left\{\vx,\alpha \vx\right\}$ is the Leaky-ReLU activation function applied element-wise, parameterized by $0 < \alpha < 1$ and $\cdot$ denotes dot product.\footnote{We do not introduce a bias term, but all our results extend to using bias (see Supplementary for a formal justification).} 

It is easy to see that such a network is as expressive as a standard two-layer network where the second layer vector is not fixed \citep{brutzkus2018sgd}. Furthermore, the assumption that the second layer is fixed is common in previous works \citep[e.g., see][]{du2018gradient,brutzkus2018sgd,ji2019polylogarithmic}. We denote row $i$ of $\mW$ by $\wvec{i}$ and row $k+i$ by $\uvec{i}$ for $1 \le i \le k$. We say that $\wvec{i}$ are the $\vw$ neurons and $\uvec{i}$ are the $\vu$ neurons. Then the network is given by: 
\begin{equation}
    \label{eq:network}
     N_{\mW}(\vx) = v\sum_{i=1}^k{\sigma\left(\wvec{i}\cdot\vx\right)}-v\sum_{i=1}^k{\sigma\left(\uvec{i}\cdot\vx\right)}
\end{equation}

{\bf Training Loss:} Define the empirical loss over $\sS$ to be the cross-entropy loss:
\begin{equation}
\label{def:Mean Empirical Loss}
    L_{\sS}(\mW) \coloneqq \frac{1}{n} \sum_{i=1}^{n} \ell(y_iN_{\mW}(\vx_{i}))
\end{equation}
where $\ell\left(z\right) = \log\left(1+e^{-z}\right)$ is the binary cross entropy loss. 

\comment{
Since we use a positive homogeneous activation (Leaky ReLU) the network we consider with $2k$ hidden neurons is as expressive as networks with
k hidden neurons and any vector $v$ in the second layer as seen at \cite{brutzkus2017sgd}. Hence, we can fix the second layer without limiting the expressive power of the two-layer network. 
Although it is relatively simpler than the case where the second layer is not fixed, the effect of over-parameterization can be studied in this setting as well.
}

{\bf Optimization Algorithm:}
The training-loss mimimization optimization problem is to find:
\begin{equation}
    \label{Minimization Objective}
    \displaystyle \arg \underset{\mW \in \R^{2k \times d}}{min} L_{\sS}(\mW)
\end{equation}
We focus on two different gradient-based methods in different parts of the paper. First, we consider the case where $L_{\sS}(\mW)$ is minimized using SGD in epochs with a batch size of one and a learning rate  $\eta$. Data points are sampled without replacement at each epoch. 
Denote by $\displaystyle \mW_{t}$ the parameters after $t$ updates.

\comment{
then the update at iteration $t$ is given by
\begin{equation}
    \label{eq:SGD weights Update Rule}
    \displaystyle
    \mW_{t} = \mW_{t-1} -\eta \frac{\partial}{\partial \mW}L_{\{(\vx_{t},y_{t})\}}(\mW_{t-1})
\end{equation}
where $L_{\{(\vx_{t},y_{t})\}}(\mW_{t-1}) = \ell\left(N_{\mW_{t-1}}(\vx_{t}),y_{t}\right)$. \alonb{Say what is the gradient update at a point that is not  differentiable.}

\alonb{Define SGD in epochs with sampling without replacement, to fit the setting of Theorem \ref{thm:Risk Convergence}}

\alonb{Don't use newline}
}

Our main optimization result, described in Section \ref{sec:risk_convergence} is shown for SGD. When studying convergence to clustered solutions, we consider gradient flow, because there we can use recent strong results from \citet{lyu2019gradient} and \citet{ji2020directional}. Recall that gradient flow is the infinitesimal step limit of gradient descent where $\mW_t$ changes continuously in time and satisfies the differential inclusion $\frac{d\mW_{t}}{dt} \in -\partial^{\circ}L_{\sS}(\mW_{t})$. Here $\partial^{\circ}L_{\sS}(\mW_{t})$ stands for Clarke's sub-differential which is a generalization of the differential for non-differentiable functions. 

\comment{
Next, we switch to optimizing the objective using gradient flow. \alonb{Justify this switch. Say that we do the switch to use Lyu's and Telgarsky's results that hold for gradient flow.} We cite the definition in \cite{lyu2019gradient}: gradient flow can be seen as gradient descent with infinitesimal step size. In this model, $\mW$ changes continuously with time , along an arc\footnote{We say that a function $z:\sI \rightarrow \R^{d}$ on the interval $\sI$ is an arc if $z$ is absolutely continuous
for any compact sub-interval of I.} $\mW_{t} : [0,+\infty) \rightarrow \R^{2kd}, t \mapsto \mW_{t}$ that satisfies the differential inclusion $\frac{d\mW_{t}}{dt} \in -\partial^{\circ}L_{\sS}(\mW_{t})$. \alonb{Is the fact that $\theta_t$ an arc an assumption? Or does it follow from the definition of gradient flow? We should understand this.}
Where the Clarke's subdifferential $\partial^{\circ}L$ is a natural generalization of the usual differential to non-differentiable functions (see Definition \ref{def:clarke's subdifferential} for the exact definition).
}
\comment{
We will need the following notation. For $\displaystyle 1 \leq i \leq k$ we define $\vw_{t}^{(i)} \in \R^{d}$ to be the incoming weight vector of neuron $i$ at iteration $t$ and denote it as a $w$ type neuron. Similarly, for $\displaystyle 1 \leq i \leq k$ we define $\vu_{t}^{(i)} \in \R^{d}$ to be the incoming weight vector of neuron $k+i$ at iteration $t$ and denote it as a $u$ type neuron.
}
The importance of gradient flow is that it can be shown to maximize margin in the following sense. Define the network margin for a single data point $(\vx_{i},y_{i})$ by $q_{i}(\mW) \coloneqq y_{i}N_{\mW}(\vx_{i})$, and the normalized network margin as:
\begin{equation}
\label{def:normalized margin}
        \overline{\gamma}(\mW) \coloneqq \frac{1}{|| \mW ||}{\min\limits_{(\vx,y) \in \sS} y_{i}N_{\mW}(\vx)}
\end{equation}
where $|| \mW ||$ is the Frobenius norm of $\mW$.

The smoothed margin is defined as:\footnote{See Remark A.4. in \citet{lyu2019gradient}.}
\begin{equation}
\label{def:smoothed normalized margin}
    \displaystyle
    \tilde{\gamma}(\mW) \coloneqq \frac{1}{|| \mW ||}{\log\left(\frac{1}{\exp(nL_{\sS}(\mW))-1}\right)}
\end{equation}
From \citet{lyu2019gradient} and \citet{ji2020directional} it follows that gradient flow converges to KKT points of the network margin maximization problem (see Supplementary for details). Here we will use this result in Section \ref{sec:exact_linear} to characterize the linear decision boundaries of learned networks.

\section{Risk Convergence}
\label{sec:risk_convergence}
We next prove that for any $\varepsilon>0$, SGD converges to $\varepsilon$ empirical-loss (see \eqref{def:Mean Empirical Loss}) within $ O\left(\frac{n^4}{\varepsilon^{2}}\right)$ updates.

Let 
$\overrightarrow{\mW}_{t}= \left(\vw_{t}^{(1)}, \dots ,\vw_{t}^{(k)}, \vu_{t}^{(1)}, \dots ,\vu_{t}^{(k)}\right) \in \R^{2kd}$ be the vectorized version of $\mW_{t}$. We assume that the network is initialized such that the norms of all rows of $\mW_{0}$ are upper bounded by some constant $R_{0}>0$. Namely for all $1 \leq i \leq k$ it holds that $\displaystyle || \vw_{0}^{(i)} ||, || \vu_{0}^{(i)} || \leq R_{0}$.

\comment{(or $\R^{2k(d+1)}$ in the case of first layer bias is included) be the vectorized version of $\mW_{t}$.}

Define $M(n,\epsilon) = \frac{C n^4}{\varepsilon^{2}} $, where $C$ is a constant that depends polynomially on $R_{x},\alpha,R_{0},k,\eta,v$ and $|| \vw^{*} ||$.\footnote{In some cases the polynomial dependence is on the inverse of the parameter, e.g., $\frac{1}{\eta}$.} See the supplementary for the exact definition of $M(n,\epsilon)$.



The following theorem states that SGD will converge to $\epsilon$ loss within $M(n,\varepsilon)$ updates.
\begin{thm}
\label{thm:Risk Convergence}
For any $\varepsilon>0$, there exists an iteration $t \le M(n,\varepsilon)$ such that $L_{\sS}(\mW_{t}) <\varepsilon$.
\end{thm}


We note that the convergence analysis holds for any $\eta > 0$. This is in line with other analyses of learning linearly separable data, which show that convergence holds for any $\eta > 0$ \citep{brutzkus2018sgd}. We next briefly sketch the proof of Theorem \ref{thm:Risk Convergence}. The full proof is deferred to the supplementary.

Our proof is based on the proof for the hinge loss in \citet{brutzkus2018sgd} with several novel ideas that enable us to show convergence for the cross entropy loss.

For the hinge loss proof, \citet{brutzkus2018sgd} consider the vector $\overrightarrow{\mW}^* = (\overbrace{\vw^* \dots \vw^*}^k, \overbrace{-\vw^* \dots -\vw^*}^k) \in \R^{2kd}$ and define $F(\mW_{t}) = \overrightarrow{\mW}_{t}\cdot \overrightarrow{\mW}^{*}$ and  $G(\mW_{t}) = \left\|\overrightarrow{\mW}_{t}\right\|$. Using an online perceptron proof and the fact that $\frac{\left|F(\mW_t)\right|}{G(\mW_t) \left\|\overrightarrow{\mW}^*\right\|} \leq 1$, they obtain a bound on the number of points with non-zero loss that SGD samples, which provides the convergence guarantee. This proof is unique to the hinge loss setting, where points can have exactly zero loss. However, in the case of the cross entropy loss, every update has a non-zero loss. Therefore, the online proof for the hinge loss cannot be applied in this case. To overcome this we (1) use an ``epoch-based'' analysis that is tailored to the SGD variant we use here, that samples data without replacement in each epoch. (2) bound the number of epochs where there exists a point with loss at least $\epsilon$. By applying these key ideas with further technical analyses that are unique to the cross entropy loss, we prove Theorem \ref{thm:Risk Convergence}.

\comment{
The analysis is based upon the convergence proof at \cite{brutzkus2017sgd} Theorem 2 with several key modification due to the different loss function (cross entropy vs hinge loss).
We consider the vector $ \overrightarrow{\mW_{t}}$ and the vector $\displaystyle \overrightarrow{\mW}^{*} = (\overbrace{\vw^{*} \dots \vw^{*}}^{k},\overbrace{-\vw^{*} \dots -\vw^{*}}^{k}) \in \R^{2kd} $. where $\langle y_{t}\vx_{t},\vw^{*} \rangle \geq 1$ holds for the points in the training set with probability 1.
We define $ F(\mW_{t}) = \langle \overrightarrow{\mW}_{t}, \overrightarrow{\mW}^{*} \rangle$ and $\displaystyle G(\mW_{t}) = || \overrightarrow{\mW}_{t} ||$.

In the case of cross entropy the notation of mistake is not the same as in hinge loss, since every point has a non-zero update, therefor we can't hope to use a standard perceptron proof.

The approach we take is to define a mistake as a time point $t$ where the loss on that time point $\ell(N_{\mW_{t-1}}(\vx_{t},y_{t}))$ is greater than $\varepsilon$.

Next we assume that there exists at least one mistake (as defined above) at every epoch until some time $T=nN_{e}$ where $N_{e}$ is the number of epochs (an epoch based approach as opposed to the online setting of the standard perceptron proof).

Under that assumption we can yield a recursive rule for lower bounding $F(\mW_{T})$ in terms of $F(\mW_{T-n})$ (time points of finishing going over an epoch), any by recursive application of the inequality we show that $F(\mW_{T})$ is lower bounded by a linear function of $T$.
Then, we give an upper bound on $\displaystyle G(\mW_{t})$ in terms
of $\displaystyle G(\mW_{t-1})$ and by a recursive application of inequalities we show that $\displaystyle G(\mW_{t})$ is bounded from above by a square root of a linear function of $T$. Finally, we use the Cauchy-Schwartz inequality, $\displaystyle \frac{| F(\mW_{t}) |}{G(\mW_{t}) || \overrightarrow{\mW}^{*} ||} \leq1$, to show that the assumption of at least one mistake at every epoch can only hold for a finite time $M(n,\varepsilon)$, therefor $\displaystyle \forall t>M(n,\varepsilon)$ the loss on any point during any epoch must be smaller than $\varepsilon$ which means  accuracy on $L_{\sS}(\mW_{t})<\varepsilon$.
}
\section{Weight Clustering and Linear Separation}
\label{sec:quasi_linear}

As shown in \figref{fig:intro fig}, learning with SGD can result in a linear decision boundary, despite the existence of zero-loss solutions that are highly non-linear. In what follows, we provide theoeretical and empirical insights into why an approximately linear boundary is learned. 

We next show a nice property of Leaky-ReLU networks that can explain why they converge to linear decision boundaries. Assume that a learned network in \eqref{eq:network} is such that all of its $\vw$ neurons 
form a ball of ``small'' radius (i.e., they are well clustered) and likewise all the $\vu$ neurons (see \figref{fig:intro fig} and  \figref{fig:Clusterization fig} for simulations that show such a case). Then, as we show in Theorem \ref{thm:difference between clustered leaky ReLus}, this implies that the resulting decision boundary will be approximately linear. Later, we give further empirical and theoretical support that learned networks indeed have this clustering structure, and together with Theorem \ref{thm:difference between clustered leaky ReLus} this explains the approximate linearity.

Consider the network in \eqref{eq:network}. Denote $\overline{\vw} = \frac{1}{k}{\sum_{i=1}^k \wvec{i}}$ and $\overline{\vu} = \frac{1}{k}{\sum_{i=1}^k \uvec{i}}$.
Also, let $\clustrad$ denote the maximum radius of the positive and negative weights around their averages. Namely:
\begin{eqnarray*}
    \|\vw^{(i)}-\overline{\vw}\|_2 &\leq& \clustrad \ \ \ i=1,\ldots,k \\
    \|\vu^{(i)}-\overline{\vu}\|_2 &\leq& \clustrad \ \ \ i=1,\ldots,k
\end{eqnarray*}
The following result says that the decision boundary will be linear except for a region whose size is determined by $\clustrad$.
\begin{thm}
\label{thm:difference between clustered leaky ReLus}
Consider the linear classifier $f(\vx) = \sgn{(\overline{\vw} - \overline{\vu})\cdot \vx}$. Then $\sgn{N_{\mW}(\vx)}=f(\vx)$ for all $\vx$ such that $|(\overline{\vw}-\overline{\vu}) \cdot \vx | \geq 2\clustrad ||\vx||$. 
\end{thm}

The theorem has a simple intuitive implication. The smaller $\clustrad$ is, the closer the classifier is to linear. In particular when $\clustrad=0$ the classifier is exactly linear.

An alternative interpretation of the theorem comes from rewriting the condition as:
\begin{equation}
    \frac{|(\overline{\vw}-\overline{\vu}) \cdot \vx |}{\|\vx\|\|\overline{\vw}-\overline{\vu}\|}\geq \frac{\clustrad}{\|\overline{\vw}-\overline{\vu}\|}
\end{equation}
Namely that linearity holds whenever the absolute value of the cosine of the angle between $\vx$ and $\overline{\vw}-\overline{\vu}$ is greater than $\frac{\clustrad}{\|\overline{\vw}-\overline{\vu}\|}$. 

The proof is in the supplementary and is somewhat technical, but a brief outline is as follows. First we show that $\forall \vx \in \R^{d} \ \text{such that} \ |\overline{\vw}\cdot \vx|  \geq r||\vx||$ either we have $[\forall 1 \leq j \leq k \ \ \vw^{(j)} \cdot \vx > 0]$ or it holds that $[\forall 1 \leq j \leq k \ \ \vw^{(j)} \cdot \vx < 0]$ and similarly for the $\vu$ neurons. Using this, we show that $\forall \vx \in \R^{d} \ \text{such that} \ |\overline{\vw}\cdot \vx|  \geq r||\vx|| \land  |\overline{\vu}\cdot \vx|  \geq r||\vx||$ it holds that $\sgn{N_{\mW}(\vx)} = \sgn{(\overline{\vw}-\overline{\vu})\cdot \vx}$. We show this by dividing the input space to four regions based on the classification of the $\vw$ and $\vu$ neurons and using properties of Leaky ReLU. Then via an involved analysis, we proceed to prove that $\sgn{N_{\mW}(\vx)} = \sgn{(\overline{\vw}-\overline{\vu})\cdot \vx}$ in other regions of the set $\left\{\vx \in \R^d \mid |(\overline{\vw}-\overline{\vu}) \cdot \vx | \geq 2\clustrad ||\vx||\right\}$, which concludes the proof.
\comment{
First we notice that our condition $\vx \in \R^{d} \ \text{ s.t } \ |(\overline{\vw}-\overline{\vu}) \cdot \vx | \geq 2\clustrad ||\vx||$ leads to $\vx \in \R^{d} \ \text{ s.t } \ |\overline{\vw} \cdot \vx | \geq \clustrad ||\vx|| \lor |\overline{\vu} \cdot \vx | \geq \clustrad ||\vx||$.

Next we show that for every option in $\vx \in \R^{d} \ \text{ s.t } \ |\overline{\vw} \cdot \vx | \geq \clustrad ||\vx|| \lor |\overline{\vu} \cdot \vx | \geq \clustrad ||\vx||$ the decision boundary of the network corresponds to a linear one, i.e., 
$\sgn{N_{\mW}(\vx)} = \sgn{(\overline{\vw}-\overline{\vu})\cdot \vx}$. holds
}

We note that the proof strongly relies on two assumptions. The first is that the activation function is Leaky ReLU. The result is not true for ReLU networks (see supplementary for an example). The second is that the clusters correspond to the $\vw$ and $\vu$ sets of neurons.



\subsection{Experiments}
\label{sec:cluster_empirical}
\begin{figure*}[tbp]
    \centering
    \begin{subfigure}[b]{0.475\textwidth}
        \centering
        \includegraphics[scale=0.5]{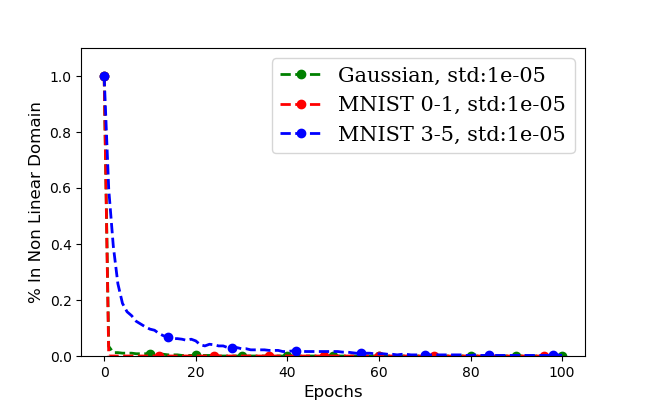}
        \caption[]%
        {}    
        \label{fig:Non Linear data points ratio - small initialization}
    \end{subfigure}
    \hfill
    \begin{subfigure}[b]{0.475\textwidth}  
        \centering 
        \includegraphics[scale=0.5]{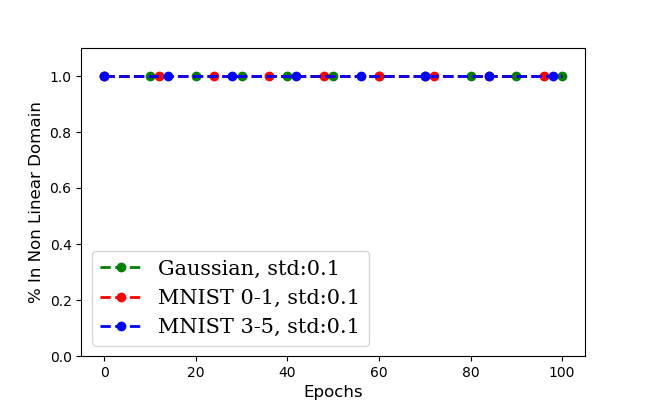}
        \caption[]%
        {}    
        \label{fig:Non Linear data points ratio - large initialization}
    \end{subfigure}
    \caption[ ]{\small{
    Empirical evaluation of the linear decision boundary prediction of Theorem \ref{thm:difference between clustered leaky ReLus}. A network is trained on Gaussian data and binary MNIST problems. At each epoch, the clustering level $r$ is calculated, and the corresponding linear decision region in Theorem \ref{thm:difference between clustered leaky ReLus}. Finally, all data points (train and test) are checked to see if they are in the linear region or not. The figure reports the fraction of points in the non-linear region. It can be seen that (a) for small initialization the fraction quickly decreases to zero whereas (b) for large initialization scale it does not.}}
    \label{fig:Clusterization fig}
\end{figure*}
Theorem \ref{thm:difference between clustered leaky ReLus} states that if neurons cluster, the resulting decision boundary will be approximately linear. But do neurons actually cluster in practice, and what is the resulting $\clustrad$? In \figref{fig:Clusterization fig} we show the value of $\clustrad$ during training. We use this $\clustrad$ to calculate the linear regime in Theorem \ref{thm:difference between clustered leaky ReLus} and the fraction of train and test points that fall outside this regime. It can be seen that for small initialization, this fraction converges to zero, implying that the learned classifiers are effectively linear over the data. Additional experiments in the supplementary provide support for the neurons being tightly clustered and $\clustrad$ being very small.

Theorem \ref{thm:difference between clustered leaky ReLus} shows that a well clustered network leads to a linear decision boundary. However, it does {\em not} imply that the network output itself is a linear function of the input. \figref{fig:ReLU vs Leaky ReLU output} provides a nice illustration of this fact. 


\comment{
In \figref{fig:Clusterization fig} we report results indicating that this is the case. 

examine the clusterization for various datasets.
The first figure \figref{fig:Non Linear data points ratio} we measure the percentage of the data points which violate our margin condition. We can see that at convergence the vast majority of data points are in the linear regime. In the second figure \figref{fig:Minimal angle with separator} we see the clusterization of the neurons, the ratio $\frac{\clustrad}{||\overline{\vw}-\overline{\vu}||}$ is approximately zero, and therefore the only directions which are not necessarily in the linear domain are only those which are nearly orthogonal to the separator $\overline{\vw}-\overline{\vu}$
}

\section{On Conditions for Convergence to Clustered Solutions}
\label{sec:exact_linear}

Figure \ref{fig:intro fig} suggests that  gradient methods converge to a network with a linear decision boundary when trained on linearly separable data. Understanding when this occurs is important, because a model with a linear decision boundary has good generalization guarantees.\footnote{For example, standard VC bounds imply $O(\sqrt{d/n})$ sample complexity in this case.}  

In the previous section we saw that clustering of neurons to two directions implies that the network has an approximate linear decision boundary.  Therefore, this reduces the problem of proving that the network has a linear decision boundary to proving that the network neurons are well clustered. It remains to show under which conditions gradient methods converge to clustered solutions.


Providing an end-to-end analysis which shows that gradient methods converge to clustered solutions is a major challenge. In this section we provide initial results for tackling this problem. In Section \ref{sec:linear_condition} we derive a novel condition on the optimization trajectory which implies that the network converges to a clustered solution and therefore to a linear decision boundary. In Section \ref{sec:perfect} we study a special case where a more fine-grained characterization of the linear decision boundary can be derived using a convex optimization program. Finally, we empirically validate our findings in Section \ref{sec:empirical_nar}.

To obtain the results in this section, we apply recent results of \citet{lyu2019gradient} and \cite{ji2020directional} and therefore make the same assumptions presented in these papers. Specifically, we assume that we run \textbf{gradient flow} (GF) as defined in Section \ref{sec:formulation}. We further assume that we are in the late phase of training:
\begin{assumption}
\label{assump:late_phase}
There exists $t_0$ such that $L_{\sS}(\mW_{t_0}) < \frac{1}{n}$.
\end{assumption}
We note that by the results in Section \ref{sec:risk_convergence}, SGD can attain the loss value in Assumption \ref{assump:late_phase}. However, in this section we need this assumption because we consider gradient flow and not SGD.

\subsection{A Sufficient Condition}
\label{sec:linear_condition}
We first observe that using Theorem \ref{thm:difference between clustered leaky ReLus} we can conclude that when the neurons are perfectly clustered around two directions (i.e., $\clustrad=0$), the decision boundary is linear. We formally define this below.
\begin{defn}
\label{def:aligned state}
    A network $N_{\mW}(\vx)$
    is \textit{perfectly clustered} if for all $1\leq i,j \leq k$ it holds that: $\vw^{(i)} =  \vw^{(j)}$
    and $\vu^{(i)} =  \vu^{(j)}$.
\end{defn}
By applying Theorem \ref{thm:difference between clustered leaky ReLus} with $\clustrad = 0$, we have:
\begin{cor}
\label{cor:perfectly_clustered_linear}
If a network $N_{\mW}$ is perfectly clustered, then its decision boundary is linear for all $\vx\in\R^d$.
\end{cor}
For completeness we provide a proof in the supplementary (this result is  easier to prove directly than Theorem \ref{thm:difference between clustered leaky ReLus}). 

The key question that remains is under which conditions is the learned network perfectly clustered? To address this, we define a novel condition on the optimization trajectory that implies clustering. We define the Neural Agreement Regime (NAR) of weights of a network as follows. Informally, a network is in the NAR regime if all the $\vw$ neurons ``agree'' on the classification of the training data and likewise for the $\vu$ neurons. Classification in both cases is within a specified margin of $\beta$. Define $\overrightarrow{\mW} = (\wvec{1}, ..., \wvec{k}, \uvec{1}, ..., \uvec{k}) \in \R^{2kd}$, $\whvec{l} = \frac{{\vw^{(l)}}}{|| \vw^{(l)} ||}$ and $\uhvec{l} = \frac{{\vu^{(l)}}}{|| \vu^{(l)} ||}$. Let $N_{\hat{\mW}}$ be the network with normalized parameters $\hat{\mW} = \frac{\mW}{\left|\left| \mW \right|\right|}$. Then, NAR is defined as follows:
\begin{defn}
\label{def:NAR definition}
Let $\beta > 0$, $\vc^{\vw} \in \{-1,1\}^n$ and $\vc^{\vu} \in \{-1,1\}^n$. We define a Neural Agreement Regime (NAR) $\gN$ with parameters $\left(\beta, \vc^{\vw}, \vc^{\vu}\right)$, to be the set of all parameters $\overrightarrow{\mW}$ such that for all $\vx_i \in \sS$ and $1 \leq l \leq k$ it holds that (1) $c^{\vw}_i \whvec{l} \cdot \vx_i  \geq \beta$ and (2)  $c^{\vu}_i \uhvec{l} \cdot \vx_i  \geq \beta$.
\end{defn}
Note that the value $c^{\vw}_i$ determines the \textit{agreement} of the $\vw$ neurons on the point $\vx_i$. Indeed, if $c^{\vw}_i = 1$, then for $1 \le l \le k$ it holds that $\whvec{l} \cdot \vx_i  \geq \beta$. Similarly, $c^{\vu}_i$ determines the agreement of the $\vu$ neurons on $\vx$. 

Importantly, if a network is in an NAR then its neurons can be ``far'' from being perfectly clustered. Namely, the angles between the normalized weights of different neurons can be relatively large. Next, we show a non-trivial fact: if gradient flow enters an NAR at some time $T_{NAR}$ and stays in it, then it will converge to a perfectly clustered network.
\begin{thm}
\label{thm:Neural Alignment}
Assume that Assumption \ref{assump:late_phase} holds and consider the NAR regime $\gN$ with parameters $\left(\beta, \vc^{\vw}, \vc^{\vu}\right)$. Assume that there exists a time $T_{NAR} \ge t_0$ such that for all $t \geq T_{NAR}$ it holds that $\overrightarrow{\mW} \in \gN$. Then, gradient flow converges to a solution in $\gN$ and at convergence the network with normalized parameters $N_{\hat{\mW}}(\vx)$ is perfectly clustered.
\end{thm}
Theorem \ref{thm:Neural Alignment} says that if training is such that the trajectory enters an NAR and never leaves it, then the network will become perfectly clustered. The proof uses results from \citet{lyu2019gradient} and \citet{ji2020directional} that together guarantee convergence of gradient flow to a KKT point of a minimum norm optimization problem. The theorem then follows from a simple observation that in an NAR, the KKT conditions imply that the network is perfectly clustered. The proof is in the supplementary.

Using Corollary \ref{cor:perfectly_clustered_linear} we 
immediately obtain the following.
\begin{cor}
\label{cor:NAR to linear decision boundary}
Under the assumptions in Theorem \ref{thm:Neural Alignment}, GF converges to a network with a linear decision boundary.
\end{cor}

Therefore, we see that if a network is at an NAR from some time $T_{NAR}$, then it will converge to a solution with a linear decision boundary. 
The question that remains is whether networks indeed converge to an NAR and remain there.

\begin{figure*}[h!]
    \centering
    \begin{subfigure}[b]{0.475\textwidth}
        \centering
        \includegraphics[scale=0.45]{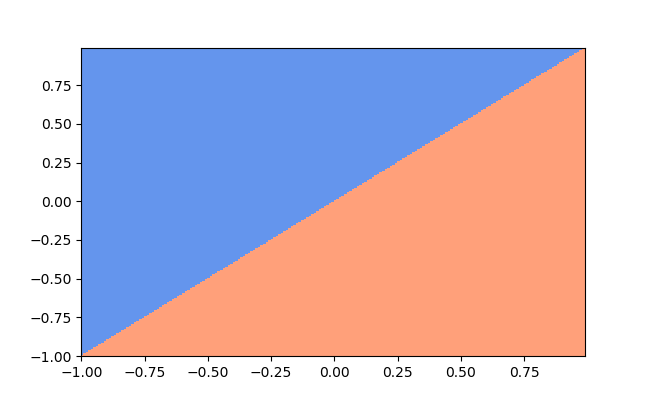}
        \caption{}
        \label{fig:Leaky ReLU network - Prediction Landscape Linear}
    \end{subfigure}
    \hfill
    \begin{subfigure}[b]{0.475\textwidth}
        \centering 
        \includegraphics[scale=0.45]{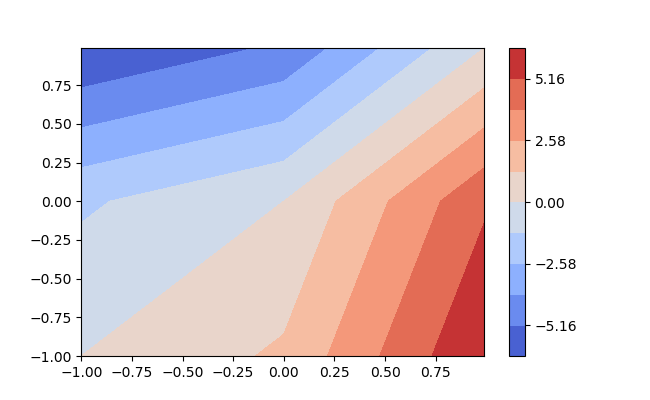}
        \caption{}
        \label{fig:Leaky ReLU network - Values Landscape Non Linear}
    \end{subfigure}
    \caption[ ]{
    The decision boundary and network output values for a two neuron leaky ReLU network with two-dimensional inputs. Figure (a) shows the decision boundary $\sgn{N_{\mW}(\vx)}$ and (b) shows the network output $N_{\mW}(\vx)$. It can be seen that the decision boundary is linear but the network output itself is not linear.}
    \label{fig:ReLU vs Leaky ReLU output}
\end{figure*}

\subsection{The Perfect Agreement Regime}
\label{sec:perfect}
To better understand convergence to NARs, in this section we study a specific NAR for which we provide a more fine-grained analysis. We identify conditions on the training data and optimization trajectory that imply that gradient flow converges to an NAR which we call the Perfect Agreement Regime (PAR). Using Theorem \ref{thm:Neural Alignment} and results from \citet{lyu2019gradient, ji2020directional}, we provide a complete characterization of the weights that gradient flow converges to in this case. Admittedly, the conditions on the data and optimization trajectory are fairly strong. Nonetheless, we show that our theoretical results accurately predict the dynamics that we observe in experiments. Indeed, in Section \ref{sec:empirical_nar} we show empirically that for certain linearly separable datasets, gradient flow converges to a solution in the PAR which is in agreement with our results.
In the PAR, each neuron classifies the data perfectly. Namely, all $\vw$ neurons classify like the ground truth $\vw^*$, and all $\vu$ neurons classify like $-\vw^*$.
Formally, let $\vy=\left(y_1,...,y_n\right)$. Then PAR is defined as follows. 
\begin{defn}[Perfect Agreement Regime]
\label{def:PAR definition}
Given training data $\sS$ with labels $\vy$, the $\mbox{PAR}(\beta)$ is the NAR with parameters $(\beta, \vy, -\vy)$.
\ignore{
We say the network \defref{Network Definition} is in the Perfect Alignment Regime (PAR) with accuracy $\beta>0$
if the network is in an NAR such that:
\begin{enumerate}
    \item $\forall \vx_{+} \in \sS_{+} : \forall 1 \leq l \leq k  \ \left(\frac{{\vw^{(l)}}}{|| \vw^{(l)} ||}\right) \cdot \vx_{+}  \geq \beta \land \left(\frac{{\vu^{(l)}}}{|| \vu^{(l)} ||}\right) \cdot \vx_{+}  \leq -\beta$
    \item $\forall \vx_{-} \in \sS_{-} : \forall 1 \leq l \leq k  \ \left(\frac{{\vw^{(l)}}}{|| \vw^{(l)} ||}\right) \cdot \vx_{-}  \leq -\beta \land \left(\frac{{\vu^{(l)}}}{|| \vu^{(l)} ||}\right) \cdot \vx_{-}  \geq \beta $
\end{enumerate}
}
\end{defn}
Note that the fact that a network is in PAR does not mean that $\vw_i=-\vu_j$. Indeed, PAR only requires that $\vw_i$ and $-\vu_j$ both correctly classify the training set.

Next, we provide conditions under which a network will converge to a PAR. The conditions require a lower bound on the network smoothed margin (\eqref{def:normalized margin}), as well as a separability condition on the data. To define the separability condition we consider the following:
\begin{align*}
    &\sV_\beta^+(\sS) \coloneqq \{ \vv \in \R^{d} \  |  \ \forall \vx \in \sS_{+} \quad \vhvec \cdot \vx \geq \beta, \\&  \exists \vx \in \sS_{-} \quad \text{ s.t. } \vhvec \cdot \vx \geq \beta \}
\end{align*}
Namely, $\sV_\beta^+(\sS)$ is the set of vectors that classifies the positive points correctly and incorrectly classifies at least one of the negative points as a positive one, where all classifications are with margin $\beta$. Similarly we define:
\begin{align*}
    &\sV_\beta^-(\sS) \coloneqq \{ \vv \in \R^{d} \  |  \ \forall \vx \in \sS_{-} \quad \vhvec \cdot \vx \geq \beta, \\&  \exists \vx \in \sS_{+} \quad \text{ s.t. } \vhvec \cdot \vx \geq \beta \}
\end{align*}

Thus, $\sV_\beta^-(\sS)$ is the same as $\sV_\beta^+(\sS)$ but with the roles of $\sS_+$ and $\sS_-$ reversed.
With these definitions we can provide a sufficient condition for convergence to PAR. 
\begin{thm}
\label{thm:multi neuron PAR}
Assume that: 
\begin{enumerate}
    \item Assumption \ref{assump:late_phase} holds.
    \item There exists an NAR $\gN$ and $T_{NAR} \ge t_0$ such that for all $t \geq T_{NAR}$ it holds that $\overrightarrow{\mW}_t \in \gN$.
    \item There exists $T_{Margin} \ge T_{NAR}$ such that $\tilde{\gamma}_{_{T_{Margin}}}  > \sqrt{k} \alpha v \cdot \max\limits_{\vx \in \sS} || x_{i} ||$
    \item The training data $\sS$ satisfies $\sV_\beta^+(\sS)= \sV_\beta^-(\sS)=\emptyset$.    
\end{enumerate}
Then $\gN$ is a $\mbox{PAR}(\beta)$ for all $t>T_{Margin}$, and there exists $\delta_w, \delta_u > 0$ such that  gradient flow converges to a network whose normalized version is perfectly clustered with neuron directions $\hat{\vw},\hat{\vu}$, where $\left(\delta_w\hat{\vw},\delta_u\hat{\vu}\right)$ is the solution to the following convex optimization problem:
\begin{align}
\label{eq:pcr_svm}
    \underset{\vw \in \R^{d}, \vu \in \R^{d}}{\argmin} \quad  || \vw & ||^{2} + || \vu ||^{2} 
     \\
    \forall \vx_{+} \in \sS_{+}:  \vw \cdot \vx_{+} & - \alpha \vu \cdot \vx_{+} \geq 1
    \nonumber \\
    \forall \vx_{-} \in \sS_{-}: \vu \cdot \vx_{-} & - \alpha \vw \cdot \vx_{-} \geq 1 \nonumber
\end{align}
\end{thm}
\begin{figure*}[th]
    \centering
    \begin{subfigure}[b]{0.475\textwidth}
        \centering
        \includegraphics[scale=0.5]{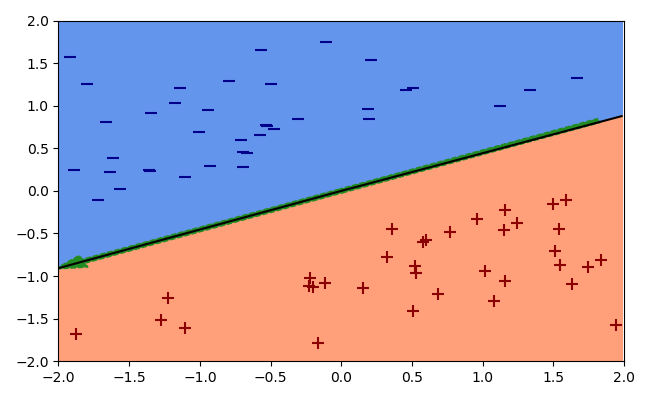}
        \caption[]%
        {{\small Learned Decision Boundary and PAR Solution}}    
        \label{fig:PAR landscape}
    \end{subfigure}
    \hfill
    \begin{subfigure}[b]{0.475\textwidth}  
        \centering 
        \includegraphics[scale=0.5]{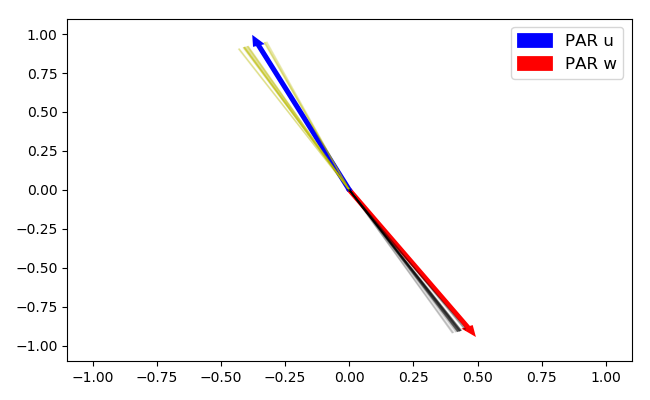}
        \caption[]%
        {{\small Learned Neurons and PAR Solutions}}    
        \label{fig:PAR neurons}
    \end{subfigure}
    \caption[ ]
    {\small{
    Illustration of the Perfect Agreement Regime (PAR):  A network with $100$ neurons is trained on linearly separable data sampled from two Gaussians, and points inside a linear margin are excluded. Figure (a) shows that the network learns a linear decision boundary. Furthermore, the green arrow shows the decision boundary predicted by the PAR result (optimization problem \eqref{eq:pcr_svm}), and it agrees with the learned boundary. Figure (b) shows the learned neurons (yellow lines for $\vu$ neurons and grey lines for $\vw$ neurons), as well as the theoretical PAR solutions. It can be seen that neurons indeed converge to the PAR solution. 
    }}
    \label{fig:pcr_fig}
\end{figure*}

We first comment on the assumptions. The first two assumptions are the same assumptions on the optimization trajectory as in Theorem \ref{thm:Neural Alignment}. Assumption 3 is another assumption on the trajectory that says that sufficiently large smoothed margin is achieved at some stage of the optimization. We note that the lower bound on the smoothed margin can be made small by considering a small $\alpha$.

Assumption 4 refers to the training set. Informally, it  corresponds to requiring that the two classes are approximately symmetric with respect to the origin. The next lemma shows that a certain symmetric training set satisfies Assumption 4:
\begin{lem}
\label{lem:symmetric}
Assume that for any $\vx \in \sS$ it holds that $-\vx \in \sS$. Then, for any $\beta > 0$, $\sV_\beta^+(\sS)= \sV_\beta^-(\sS)=\emptyset$.
\end{lem}
The proof is given in the supplementary. This example suggests that we should observe PAR in symmetric distributions, which produce approximately symmetric training sets. Indeed, we empirically show in Section \ref{sec:empirical_nar} that gradient flow converges to a solution in PAR for a distribution with two symmetric Gaussians. We note that this example shows that Assumption 4 is independent of the maximum margin attainable on the training set. Indeed, by scaling the points, we can obtain any margin and still satisfy the assumption.

The theorem not only implies that convergence will be to a PAR, but it provides the solution that GF will converge to. The optimization problem in \eqref{eq:pcr_svm} is an SVM optimization problem with the kernel: $K(\vx,\vx') =\sum_{y\in\{-1,1\}} \sigma'(y\vw^{*} \cdot \vx)\sigma'(y\vw^{*} \cdot \vx')\vx \cdot \vx'$. The corresponding feature map is:
$\phi(\vx) = [\sigma'(\vw^{*} \cdot \vx)\vx,-\sigma'(-\vw^{*} \cdot \vx)\vx]\in \R^{2d}$.

We prove Theorem \ref{thm:multi neuron PAR} in the supplementary, and provide a sketch next. First, we use Theorem \ref{thm:Neural Alignment} to show that gradient flow converges to an NAR and the neurons are clustered. Then we show that under Assumption 3 and using the monotonicity of the smoothed margin (\eqref{def:smoothed normalized margin}), by \citet{lyu2019gradient}, all $\vw$ neurons classify the positive points correctly and all $\vu$ neurons classify the negative points correctly for all $t > T_{Margin}$. Then, using Assumption 4 we show that the solution is in PAR. Finally, we use results of \citet{lyu2019gradient} to show that the network directions solve the convex optimization problem in the theorem.


\subsection{Experiments}
\label{sec:empirical_nar}
In Theorem \ref{thm:multi neuron PAR} we show that when learning enters the PAR regime the solution will be given by \eqref{eq:pcr_svm}. We performed experiments in several settings that show the above behavior is observed in practice when classes are sampled from Gaussians. Figure \ref{fig:pcr_fig} shows the decision boundary (Figure \ref{fig:PAR landscape}) and learned weights (Figure \ref{fig:PAR neurons}), for learning from points sampled from two classes corresponding to Gaussians. The figure also shows the PAR predictions for the decision boundary and learned weights, and these show excellent agreement with the empirical results. We have also verified that in this case convergence is indeed to a PAR solution. We performed such experiments also for higher dimensional settings, and the results are in the supplementary. Finally, note that we do not expect learning to always converge to a PAR. In the supplementary we show an example where this does not happen.

\section{Conclusions}
Optimization and generalization are closely coupled in deep-learning. Yet both are little understood even for simple models. Here we consider perhaps the simplest ``teacher'' model where the ground truth is linear. We prove that cross-entropy can be globally minimized by SGD, despite the non-convexity of the loss, and for any initialization scale. We are not aware of any such result for non-linear networks (for example NTK optimization results require large initialization scale, and sufficiently wide networks \citep{ji2019polylogarithmic}). Our novel proof technique analyzes SGD in an offline setting and uses the notion of loss-violation per epoch, which we believe could be useful elsewhere.

In our setting, small initialization scale leads empirically to approximately linear decision boundaries. We prove that such boundaries are obtained when neurons with same output-weight sign are clustered. Empirically we show that such clustering indeed occurs. Moreover, we provide sufficient conditions for converging to such clustered solutions. 

Several open questions remain. The first is reducing the assumptions when proving convergence to a clustered solution. Another interesting direction is extending our results to simple non-linear teachers. 

\section{Acknowledgements}
This research is supported by the European Research
Council (ERC) under the European Unions Horizon 2020
research and innovation programme (grant ERC HOLI
819080) and by the Yandex Initiative in Machine
Learning at Tel Aviv University. AB is supported by the Google Doctoral Fellowship in Machine Learning.
\bibliography{sources}

\begin{thebibliography}{34}
\providecommand{\natexlab}[1]{#1}
\providecommand{\url}[1]{\texttt{#1}}
\expandafter\ifx\csname urlstyle\endcsname\relax
  \providecommand{\doi}[1]{doi: #1}\else
  \providecommand{\doi}{doi: \begingroup \urlstyle{rm}\Url}\fi

\bibitem[Allen-Zhu \& Li(2019)Allen-Zhu and Li]{allen2019can}
Allen-Zhu, Z. and Li, Y.
\newblock What can resnet learn efficiently, going beyond kernels?
\newblock \emph{arXiv preprint arXiv:1905.10337}, 2019.

\bibitem[Allen-Zhu et~al.(2019)Allen-Zhu, Li, and Song]{allen2019convergence}
Allen-Zhu, Z., Li, Y., and Song, Z.
\newblock A convergence theory for deep learning via over-parameterization.
\newblock In \emph{International Conference on Machine Learning}, pp.\
  242--252. PMLR, 2019.

\bibitem[Arora et~al.(2019)Arora, Du, Hu, Li, and Wang]{arora2019fine}
Arora, S., Du, S., Hu, W., Li, Z., and Wang, R.
\newblock Fine-grained analysis of optimization and generalization for
  overparameterized two-layer neural networks.
\newblock In \emph{International Conference on Machine Learning}, pp.\
  322--332, 2019.

\bibitem[Blum \& Rivest(1992)Blum and Rivest]{blum1992training}
Blum, A.~L. and Rivest, R.~L.
\newblock Training a 3-node neural network is np-complete.
\newblock \emph{Neural Networks}, 5\penalty0 (1):\penalty0 117--127, 1992.

\bibitem[Brutzkus \& Globerson(2019)Brutzkus and Globerson]{brutzkus2019larger}
Brutzkus, A. and Globerson, A.
\newblock Why do larger models generalize better? a theoretical perspective via
  the xor problem.
\newblock In \emph{International Conference on Machine Learning}, pp.\
  822--830. PMLR, 2019.

\bibitem[Brutzkus et~al.(2018)Brutzkus, Globerson, Malach, and
  Shalev-Shwartz]{brutzkus2018sgd}
Brutzkus, A., Globerson, A., Malach, E., and Shalev-Shwartz, S.
\newblock Sgd learns over-parameterized networks that provably generalize on
  linearly separable data.
\newblock \emph{International Conference on Learning Representations}, 2018.

\bibitem[Cao \& Gu(2019)Cao and Gu]{cao2019generalization}
Cao, Y. and Gu, Q.
\newblock Generalization bounds of stochastic gradient descent for wide and
  deep neural networks.
\newblock In \emph{Advances in Neural Information Processing Systems}, pp.\
  10836--10846, 2019.

\bibitem[Chizat \& Bach(2018)Chizat and Bach]{chizat2018global}
Chizat, L. and Bach, F.
\newblock On the global convergence of gradient descent for over-parameterized
  models using optimal transport.
\newblock In \emph{Advances in neural information processing systems}, pp.\
  3036--3046, 2018.

\bibitem[Chizat \& Bach(2020)Chizat and Bach]{ChizatB20}
Chizat, L. and Bach, F.
\newblock Implicit bias of gradient descent for wide two-layer neural networks
  trained with the logistic loss.
\newblock In Abernethy, J.~D. and Agarwal, S. (eds.), \emph{Conference on
  Learning Theory, {COLT} 2020, 9-12 July 2020, Virtual Event [Graz, Austria]},
  volume 125 of \emph{Proceedings of Machine Learning Research}, pp.\
  1305--1338. {PMLR}, 2020.
\newblock URL \url{http://proceedings.mlr.press/v125/chizat20a.html}.

\bibitem[Daniely \& Malach(2020)Daniely and Malach]{daniely2020learning}
Daniely, A. and Malach, E.
\newblock Learning parities with neural networks.
\newblock \emph{arXiv preprint arXiv:2002.07400}, 2020.

\bibitem[Daniely et~al.(2016)Daniely, Frostig, and Singer]{daniely2016toward}
Daniely, A., Frostig, R., and Singer, Y.
\newblock Toward deeper understanding of neural networks: The power of
  initialization and a dual view on expressivity.
\newblock In \emph{Advances In Neural Information Processing Systems}, pp.\
  2253--2261, 2016.

\bibitem[Davis et~al.(2018)Davis, Drusvyatskiy, Kakade, and
  Lee]{davis2018stochastic}
Davis, D., Drusvyatskiy, D., Kakade, S., and Lee, J.~D.
\newblock Stochastic subgradient method converges on tame functions, 2018.

\bibitem[Devlin et~al.(2019)Devlin, Chang, Lee, and Toutanova]{devlin2019bert}
Devlin, J., Chang, M.-W., Lee, K., and Toutanova, K.
\newblock Bert: Pre-training of deep bidirectional transformers for language
  understanding.
\newblock In \emph{NAACL-HLT (1)}, 2019.

\bibitem[Du et~al.(2019)Du, Lee, Li, Wang, and Zhai]{du2018gradient2}
Du, S., Lee, J., Li, H., Wang, L., and Zhai, X.
\newblock Gradient descent finds global minima of deep neural networks.
\newblock In \emph{International Conference on Machine Learning}, pp.\
  1675--1685, 2019.

\bibitem[Du et~al.(2018)Du, Zhai, Poczos, and Singh]{du2018gradient}
Du, S.~S., Zhai, X., Poczos, B., and Singh, A.
\newblock Gradient descent provably optimizes over-parameterized neural
  networks.
\newblock \emph{International Conference on Learning Representations}, 2018.

\bibitem[Fiat et~al.(2019)Fiat, Malach, and Shalev-Shwartz]{fiat2019decoupling}
Fiat, J., Malach, E., and Shalev-Shwartz, S.
\newblock Decoupling gating from linearity.
\newblock \emph{arXiv preprint arXiv:1906.05032}, 2019.

\bibitem[Gunasekar et~al.(2018)Gunasekar, Lee, Soudry, and
  Srebro]{gunasekar2018implicit}
Gunasekar, S., Lee, J.~D., Soudry, D., and Srebro, N.
\newblock Implicit bias of gradient descent on linear convolutional networks.
\newblock In \emph{Advances in Neural Information Processing Systems}, pp.\
  9461--9471, 2018.

\bibitem[Jacot et~al.(2018)Jacot, Gabriel, and Hongler]{jacot2018neural}
Jacot, A., Gabriel, F., and Hongler, C.
\newblock Neural tangent kernel: Convergence and generalization in neural
  networks.
\newblock In \emph{Advances in neural information processing systems}, pp.\
  8571--8580, 2018.

\bibitem[Ji \& Telgarsky(2019{\natexlab{a}})Ji and Telgarsky]{ji2018gradient}
Ji, Z. and Telgarsky, M.
\newblock Gradient descent aligns the layers of deep linear networks.
\newblock \emph{ICLR}, 2019{\natexlab{a}}.

\bibitem[Ji \& Telgarsky(2019{\natexlab{b}})Ji and
  Telgarsky]{ji2019polylogarithmic}
Ji, Z. and Telgarsky, M.
\newblock Polylogarithmic width suffices for gradient descent to achieve
  arbitrarily small test error with shallow relu networks.
\newblock In \emph{International Conference on Learning Representations},
  2019{\natexlab{b}}.

\bibitem[Ji \& Telgarsky(2020)Ji and Telgarsky]{ji2020directional}
Ji, Z. and Telgarsky, M.
\newblock Directional convergence and alignment in deep learning, 2020.

\bibitem[Krizhevsky et~al.(2012)Krizhevsky, Sutskever, and
  Hinton]{krizhevsky2012imagenet}
Krizhevsky, A., Sutskever, I., and Hinton, G.~E.
\newblock Imagenet classification with deep convolutional neural networks.
\newblock In \emph{Advances in neural information processing systems}, pp.\
  1097--1105, 2012.

\bibitem[Li \& Liang(2018)Li and Liang]{li2018learning}
Li, Y. and Liang, Y.
\newblock Learning overparameterized neural networks via stochastic gradient
  descent on structured data.
\newblock In \emph{Advances in Neural Information Processing Systems}, pp.\
  8157--8166, 2018.

\bibitem[Li et~al.(2020)Li, Ma, and Zhang]{li2020learning}
Li, Y., Ma, T., and Zhang, H.~R.
\newblock Learning over-parametrized two-layer neural networks beyond {NTK}.
\newblock In \emph{Conference on Learning Theory}, pp.\  2613--2682, 2020.

\bibitem[Lyu \& Li(2020)Lyu and Li]{lyu2019gradient}
Lyu, K. and Li, J.
\newblock Gradient descent maximizes the margin of homogeneous neural networks.
\newblock \emph{ICLR}, 2020.

\bibitem[Mei et~al.(2018)Mei, Montanari, and Nguyen]{mei2018mean}
Mei, S., Montanari, A., and Nguyen, P.-M.
\newblock A mean field view of the landscape of two-layer neural networks.
\newblock \emph{Proceedings of the National Academy of Sciences}, 115\penalty0
  (33):\penalty0 E7665--E7671, 2018.

\bibitem[Moroshko et~al.(2020)Moroshko, Gunasekar, Woodworth, Lee, Srebro, and
  Soudry]{moroshko2020implicit}
Moroshko, E., Gunasekar, S., Woodworth, B., Lee, J.~D., Srebro, N., and Soudry,
  D.
\newblock Implicit bias in deep linear classification: Initialization scale vs
  training accuracy.
\newblock \emph{arXiv preprint arXiv:2007.06738}, 2020.

\bibitem[Nacson et~al.(2019)Nacson, Gunasekar, Lee, Srebro, and
  Soudry]{nacson2019lexicographic}
Nacson, M.~S., Gunasekar, S., Lee, J., Srebro, N., and Soudry, D.
\newblock Lexicographic and depth-sensitive margins in homogeneous and
  non-homogeneous deep models.
\newblock In \emph{International Conference on Machine Learning}, pp.\
  4683--4692, 2019.

\bibitem[Phuong \& Lampert(2021)Phuong and Lampert]{lampert21_orthogonal}
Phuong, M. and Lampert, C.
\newblock The inductive bias of relu networks on orthogonally separable data.
\newblock \emph{ICLR}, 2021.

\bibitem[Silver et~al.(2016)Silver, Huang, Maddison, Guez, Sifre, Van
  Den~Driessche, Schrittwieser, Antonoglou, Panneershelvam, Lanctot,
  et~al.]{silver2016mastering}
Silver, D., Huang, A., Maddison, C.~J., Guez, A., Sifre, L., Van Den~Driessche,
  G., Schrittwieser, J., Antonoglou, I., Panneershelvam, V., Lanctot, M.,
  et~al.
\newblock Mastering the game of go with deep neural networks and tree search.
\newblock \emph{nature}, 529\penalty0 (7587):\penalty0 484, 2016.

\bibitem[Wang et~al.(2019)Wang, Giannakis, and Chen]{Wang_2019}
Wang, G., Giannakis, G.~B., and Chen, J.
\newblock Learning relu networks on linearly separable data: Algorithm,
  optimality, and generalization.
\newblock \emph{IEEE Transactions on Signal Processing}, 67\penalty0
  (9):\penalty0 2357–2370, May 2019.
\newblock ISSN 1941-0476.
\newblock \doi{10.1109/tsp.2019.2904921}.
\newblock URL \url{http://dx.doi.org/10.1109/TSP.2019.2904921}.

\bibitem[Wei et~al.(2019)Wei, Lee, Liu, and Ma]{wei2019regularization}
Wei, C., Lee, J.~D., Liu, Q., and Ma, T.
\newblock Regularization matters: Generalization and optimization of neural
  nets vs their induced kernel.
\newblock In \emph{Advances in Neural Information Processing Systems}, pp.\
  9712--9724, 2019.

\bibitem[Woodworth et~al.(2020)Woodworth, Gunasekar, Lee, Moroshko, Savarese,
  Golan, Soudry, and Srebro]{WoodworthGLMSGS20}
Woodworth, B.~E., Gunasekar, S., Lee, J.~D., Moroshko, E., Savarese, P., Golan,
  I., Soudry, D., and Srebro, N.
\newblock Kernel and rich regimes in overparametrized models.
\newblock In Abernethy, J.~D. and Agarwal, S. (eds.), \emph{Conference on
  Learning Theory, {COLT} 2020, 9-12 July 2020, Virtual Event [Graz, Austria]},
  volume 125 of \emph{Proceedings of Machine Learning Research}, pp.\
  3635--3673. {PMLR}, 2020.
\newblock URL \url{http://proceedings.mlr.press/v125/woodworth20a.html}.

\bibitem[Yehudai \& Shamir(2019)Yehudai and Shamir]{yehudai2019power}
Yehudai, G. and Shamir, O.
\newblock On the power and limitations of random features for understanding
  neural networks.
\newblock In \emph{Advances in Neural Information Processing Systems}, pp.\
  6594--6604, 2019.

\end{thebibliography}
\bibliographystyle{icml2020}
\appendix
\onecolumn







\onecolumn
\icmltitle{Towards Understanding Neural Networks with Linear Teachers \\ \textit{Supplementary Material}}




\section{Gradient Flow Definitions}
\label{sec:formulation}

We next formally define gradient flow. A function $f:\sX \rightarrow \R$ is locally Lipschitz if for every $\vx \in \sX$ there exists a neighborhood $\sU$ of $\vx$ such that the restriction of $f$ on $\sU$ is Lipschitz continuous. For a locally Lipschitz function $f: \sX \rightarrow \R$, the Clarke subdifferential at $\vx \in \sX$ is the convex set:
\begin{equation}
\label{def:clarke's subdifferential}
    \partial^{\circ}f(\vx) \coloneqq {\rm conv}\left\{\lim\limits_{k \rightarrow \infty} \nabla f(\vx_{k}) : \vx_{k} \rightarrow \vx,f \text{ is differentiable at } \vx_{k}\right\}
\end{equation}

As in \cite{lyu2019gradient} and \cite{ji2020directional}, a curve $z$ from an interval $I$ to a real space $\R^{m}$ is called an arc if it is absolutely continuous on any compact subinterval of $I$. For an arc $z$ we use $z'(t)$ (or $\frac{dz}{dt}(t))$ to denote the derivative at $t$ if it
exists. We say that a locally Lipschitz function $f : \R^{d} \rightarrow \R$ admits a chain rule if for any arc $z : [0;+\infty) \rightarrow \R^{d}, \forall h \in \partial^{\circ}f(z(t)) : (f \circ z)'(t) =\langle h, z'(t) \rangle$ holds for a.e. $t \ge 0$. It holds that an arc is a.e. differentiable, and the composition of an arc and a locally Lipschitz function is still an arc. 

Given the definitions above, we define gradient flow $\mW: [0,\infty) \rightarrow \R^{k}$ to be an arc that satisfies the following differential inclusion for a.e. $t \ge 0$:
\begin{equation}
    \label{eq:gradient flow definition}
    \frac{d \mW_{t}}{dt} \in -\partial^{\circ}L_{\sS}(\mW_{t})
\end{equation}

\section{Proof of Theorem \ref{thm:Risk Convergence}}
\label{Risk Convergence Proof}
Throughout this proof we will sometimes use the notation $\langle \vx,\vy \rangle$ as the dot product between two vectors $\vx$ and $\vy$ for readability purposes.

Let  $\displaystyle \overrightarrow{\mW}^{*} = (\overbrace{\vw^{*} \dots \vw^{*}}^{k},\overbrace{-\vw^{*} \dots -\vw^{*}}^{k}) \in \R^{2kd} $.

Define the following two functions:
$$\displaystyle F(\mW_{t}) = \langle \overrightarrow{\mW}_{t}, \overrightarrow{\mW}^{*} \rangle = \sum\limits_{i=1}^{k} \langle \vw_{t}^{(i)},\vw^{*} \rangle - \sum\limits_{i=1}^{k} \langle \vu_{t}^{(i)},\vw^{*} \rangle $$ and $$\displaystyle G(\mW_{t}) = || \overrightarrow{\mW}_{t} || = \sqrt{\sum\limits_{i=1}^{k} || \vw_{t}^{(i)} ||^2 + \sum\limits_{i=1}^{k} || \vu_{t}^{(i)} ||^2}$$
Then, from Cauchy-Schwartz inequality we have:
\begin{equation}
    \label{cauchy-schwartz}
    \displaystyle
    \frac{\vert F(\mW_{t}) \vert}{G(\mW_{t}) ||\overrightarrow{\mW}^{*} ||} = \frac{\vert \langle \overrightarrow{\mW}_{t}, \overrightarrow{\mW}^{*} \rangle \vert}{|| \overrightarrow{\mW}_{t} ||   ||\overrightarrow{\mW}^{*} ||} \leq 1
\end{equation}

Recall we define: $N_{\mW}(\vx) = v \sum\limits_{j=1}^{k} \sigma (\vw^{(j)} \cdot\vx) - v\sum\limits_{j=1}^{k}\sigma (\vu^{(j)} \cdot\vx)$. 

We consider minimizing the objective function:
$$L_{\sS}(\mW) = \frac{1}{n}\sum\limits_{i=1}^{n}\log\left(1+e^{-y_{i}N_{\mW}(\vx_{i})}\right)$$
using SGD on $\sS$ where each point is sampled without replacement at each epoch. WLOG, we set $\sigma'(0) = \alpha$.

We first outline the proof structure. Let's assume we run SGD for $N_{e}$ epochs and denote $T=nN_{e}$. Furthermore, we assume that for all epochs up to this point there is at least one point in the epoch s.t. $\ell(y_{t}N_{\mW_{t-1}}(\vx_{t})) > \varepsilon_{0}$ for some $\varepsilon_{0} >0$ (recall that $n$ is the number of training points, and $(y_{t},\vx_{t})$ is some training point selected during some epoch).

First, we will show that after at most $T \le M(n,\epsilon_0)$ iterations, there exists an epoch $i_e$ such that for each point $(\vx_t,y_t) \in \sS$ sampled in the epoch, it holds that:

\begin{equation}
    \label{eq:First half of risk convergence proof}
    \ell(y_{t}N_{\mW_{t-1}}(\vx_{t})) \leq \varepsilon_{0}
\end{equation}

Next, using the Lipschitzness of $\ell(x)$ we will show that the loss on points cannot change too much during an epoch. Specifically, we will use this to show that at the end of epoch $i_e$, which we denote by time $T^*$, it holds for all $(\vx_i,y_i) \in \sS$:
\begin{equation}
    \label{eq:Second half of risk convergence proof}
    \ell(y_{i}N_{\mW_{T^{*}}}(\vx_{i})) \leq (1+2v^{2}R_{x}^{2}\eta k n )\varepsilon_{0}
\end{equation}
now by choosing $\varepsilon_{0} = \frac{\varepsilon}{1+2v^{2}R_{x}^{2}\eta k n}$ we will get that $\forall 1 \leq i \leq n \ \ell(y_{i}N_{\mW_{T^{*}}}(\vx_{i})) \leq \varepsilon$ which shows that $L_{\sS}(\mW_{T^{*}}) \leq \varepsilon$ as required.

We start by showing \eqref{eq:First half of risk convergence proof}.

For the gradient of each neuron we have:

\begin{align*}
    \frac{\partial L_{\{(\vx_{i},y_{i})\}}(\mW)}{\partial \vw^{(j)}} &= \frac{e^{-y_{i}N_{\mW}(\vx_{i})}}{1+e^{-y_{i}N_{\mW}(\vx_{i})}} \cdot \frac{-y_{i}\partial N_{\mW}(\vx_{i})}{\partial \vw^{(j)}} \\ &= \frac{-y_{i}e^{-y_{i}N_{\mW}(\vx_{i})}}{1+e^{-y_{i}N_{\mW}(\vx_{i})}} \cdot v\vx_{i}\sigma'(\vw^{(j)} \cdot\vx_{i}) \\ &= -vy_{i}\vx_{i} \left \vert \ell'(y_{i}N_{\mW}(\vx_{i})) \right \vert \sigma'(\vw^{(j)} \cdot\vx_{i})
\end{align*}

and similarly:
$$\frac{\partial L_{\{(\vx_{i},y_{i})\}}(\mW)}{\partial \vu^{(j)}}=vy_{i}\vx_{i} \left \vert \ell'(y_{i}N_{\mW}(\vx_{i})) \right \vert \sigma'(\vu^{(j)} \cdot\vx_{i})$$
where $\ell'(x) = -\frac{e^{-x}}{1+e^{-x}}=-\frac{1}{1+e^{x}}$ and $\ell(x) = log(1+e^{-x})$.

Optimizing by SGD yields the following update rule:
$$\mW_{t} = \mW_{t-1} - \eta \frac{\partial}{\partial \mW}L_{\{(\vx_{t},y_{t})\}}(\mW_{t-1})$$
where $\mW_{t} = (\vw_{t}^{(1)},...,\vw_{t}^{(k)},\vu_{t}^{(1)},...,\vu_{t}^{(k)})$.

For every neuron we get the following updates:

\begin{enumerate}
\label{eq:neurons updates}
    \item $\vw_{t}^{(j)} = \vw_{t-1}^{(j)} +\eta vy_{t}\vx_{t}\left \vert \ell'(y_{t}N_{\mW_{t-1}}(\vx_{t})) \right \vert p_{t-1}^{(j)}$
    \item $\vu_{t}^{(j)} = \vu_{t-1}^{(j)} - \eta vy_{t}\vx_{t}\left \vert \ell'(y_{t}N_{\mW_{t-1}}(\vx_{t})) \right \vert q_{t-1}^{(j)}$
\end{enumerate}
where $p_{t}^{(j)} := \sigma'(\vw_{t}^{(j)} \cdot \vx_{t+1}) ;q_{t}^{(j)} := \sigma'(\vu_{t}^{(j)} \cdot \vx_{t+1})$.

Next we will show recursive upper bounds for $\displaystyle G(\mW_{t})$ and $\displaystyle F(\mW_{t})$.
\begin{align*}
    &G(\mW_{t})^{2} = \sum\limits_{j=1}^{k} || \vw_{t}^{(j)} ||^{2} +\sum\limits_{j=1}^{k} || \vu_{t}^{(j)} ||^{2}   \\ &\leq  \sum\limits_{j=1}^{k} || \vw_{t-1}^{(j)} ||^{2} +\sum\limits_{j=1}^{k} || \vu_{t-1}^{(j)} ||^{2} \\&+ 2\eta y_{t}\vert \ell'(y_{t}N_{\mW_{t-1}}(\vx_{t})) \vert \left(\sum\limits_{j=1}^{k} \langle \vw_{t-1}^{(j)},\vx_{t}\rangle p_{t-1}^{(j)}v - \sum\limits_{j=1}^{k} \langle \vu_{t-1}^{(j)},\vx_t\rangle q_{t-1}^{(j)}v\right) \\ & +2k\eta^{2}v^{2}|| \vx_{t} ||^{2}\vert \ell'(y_{t}N_{\mW_{t-1}}(\vx_{t}))\vert ^2 = \sum\limits_{j=1}^{k} || \vw_{t-1}^{(j)} ||^{2} +\sum\limits_{j=1}^{k} || \vu_{t-1}^{(j)} ||^{2}  \\ & +2\eta \vert \ell'(y_{t}N_{\mW_{t-1}}(\vx_{t}))\vert y_{t}N_{\mW_{t-1}}(\vx_{t}) +  2k\eta ^{2} v^{2} || \vx_{t}  || ^{2}\vert \ell'(y_{t}N_{\mW_{t-1}}(\vx_{t})) \vert^2 \\ &= G(\mW_{t-1})^{2} +2\eta \vert \ell'(y_{t}N_{\mW_{t-1}}(\vx_{t})) \vert y_{t}N_{\mW_{t-1}}(\vx_{t}) + 2k\eta^{2}v^{2}|| \vx_{t} ||^2 \vert \ell'(y_{t}N_{\mW_{t-1}}(\vx_{t}))\vert ^2
\end{align*}
On the other hand,

\begin{align*}
&F(\mW_{t}) =  \sum\limits_{j=1}^{k}\langle \vw_{t}^{(j)},\vw^{*}\rangle - \sum\limits_{j=1}^{k} \langle\vu_{t}^{(j)},\vw^{*}\rangle = \sum\limits_{j=1}^{k}\langle \vw_{t-1}^{(j)},\vw^{*}\rangle - \sum\limits_{j=1}^{k} \langle \vu_{t-1}^{(j)},\vw^{*}\rangle \\ & +\eta \vert \ell'(y_{t}N_{W_{t-1}}(\vx_{t}))\vert \sum\limits_{j=1}^{k} \langle y_{t}\vx_{t},\vw^{*}\rangle p_{t-1}^{(j)}v + \eta \vert \ell'(y_{t}N_{W_{t-1}}(\vx_{t}))\vert \sum\limits_{j=1}^{k} \langle y_{t}\vx_{t},\vw^{*}\rangle q_{t-1}^{(j)}v 
\\ & \geq \sum\limits_{j=1}^{k} \langle \vw_{t-1}^{(j)},\vw^{*}\rangle - \sum\limits_{j=1}^{k} \langle \vu_{t-1}^{(j)},\vw^{*}\rangle  +2k\eta  v\alpha  \vert \ell'(y_{t}N_{W_{t-1}}(\vx_{t})) \vert
\end{align*}

Where we used the inequalities $\displaystyle \langle y_{t}\vx_{t},\vw^{*} \rangle \geq 1$ and $q_{t}^{(j)},p_{t}^{(j)} \geq \alpha$.

To summarize we have:

\begin{equation}
    \label{G Recursion rule}
    \displaystyle G(\mW_{t})^{2} \leq G(\mW_{t-1})^{2} +2 \eta \vert \ell'(y_{t}N_{\mW_{t-1}}(\vx_{t})) \vert y_{t}N_{\mW_{t-1}}(\vx_{t}) + 2k\eta^{2}v^{2}R_{x}^2 \vert \ell'(y_{t}N_{\mW_{t-1}}(\vx_{t})) \vert^2
\end{equation} 
\begin{equation}
    \label{F Recursion rule}
    \displaystyle
    F(\mW_{t}) \geq F(\mW_{t-1}) +2k\eta v \alpha \vert  \ell'(y_{t}N_{\mW_{t-1}}(\vx_{t})) \vert
\end{equation}
For an upper bound on $G(\mW_{t})$ we use the following inequalities (which hold for the cross entropy loss): 

$\forall x \in \R \quad \frac{x}{1+e^{x}} \leq 1 \Rightarrow  \vert \ell'(y_{t}N_{\mW_{t-1}}(\vx_{t})) \vert y_{t}N_{\mW_{t-1}}(\vx_{t}) =\frac{y_{t}N_{\mW_{t-1}}(\vx_{t})}{1+e^{y_{t}N_{\mW_{t-1}}(\vx_{t})}} \leq 1$ and $\vert \ell'(y_{t}N_{\mW_{t-1}}(\vx_{t})) \vert \leq 1$.
Together we have for any $t$:
$$G(\mW_{t})^{2} \leq G(\mW_{t-1})^{2} +2 \eta  + 2k\eta^{2}v^{2}R_{x}^2$$
Using this recursively up until $T=nN_{e}$ we get:
\begin{equation}
\label{G final inequality}
  G(\mW_{T})^2 \leq G(\mW_0)^2 +T(2k\eta^{2}v^{2}R_{x}^2 +2\eta)  
\end{equation}
Now, for $F(\mW_{t})$, let $\varepsilon_{0} >0$, under our assumption, in any epoch $i_{e}$ until $N_{e}$ ($1 \leq i_{e} \leq N_{e}$) there exists at least one point in the epoch $(y_{t_{i_{e}}},\vx_{t_{i_{e}}}) \in \sS$ s.t. $\ell(y_{t_{i_{e}}}N_{\mW_{t_{i_{e}}}}(\vx_{t_{i_{e}}})) > \varepsilon_{0}$.

Now, since in our case $\ell(x)= log(1+e^{-x})$ and $\ell'(x) = -\frac{1}{1+e^{x}}$, we see that the condition $\ell(x) >  \varepsilon_{0}$ implies that:
\begin{equation}
    \label{loss and loss gradient relations}
    \displaystyle
    \vert \ell'(x)\vert >1-e^{-\varepsilon_{0}}
\end{equation}



In any other case $|\ell'(y_{t}N_{\mW_{t-1}}(\vx_{t}))| \geq 0$,
so if we assume at least one point violation per epoch (i.e. $\ell(y_{t_{i_{e}}}N_{\mW_{t_{i_{e}}}}(\vx_{t_{i_{e}}})) \geq  \varepsilon_{0}$ for some point $\left(y_{t_{i_{e}}},\vx_{t_{i_{e}}}\right)$ in the epoch) we would get that at the end of epoch $N_{e}$:
\begin{equation}
    \label{F Recursion violation rule}
    F(\mW_{T}) \geq F(\mW_{T-n})+2k\eta v \alpha (1-e^{-\varepsilon_{0}})
\end{equation}

This implies that (recursively using \eqref{F Recursion violation rule}):
\begin{equation}
\label{F final inequality}
    F(\mW_{T}) \geq F(\mW_{0}) + 2k \eta v \alpha N_{e}(1-e^{-\varepsilon_{0}})
\end{equation}
where $N_{e}$ is the number of epochs and $n$ the number of training points, $T = nN_{e}$.

Now, using the Cauchy-Schwartz, \eqref{G final inequality} and \eqref{F final inequality} we have:

\begin{align*}
\displaystyle
    &-G(\mW_0) || \overrightarrow{\mW}^{*} || +2k \eta v \alpha N_{e}(1-e^{-\varepsilon_{0}}) \leq F(\mW_0) + 2k \eta v \alpha N_{e}(1-e^{-\varepsilon_{0}}) \\ & \leq F(\mW_{T})  \leq || \overrightarrow{\mW}^{*} || G(\mW_{T}) \leq  || \overrightarrow{\mW}^{*} || \sqrt{G(\mW_0)^2 +T(2k\eta^{2}v^{2}R_{x}^2 +2\eta)}
\end{align*}

Using $\sqrt{a+b} \leq \sqrt{a}+\sqrt{b}$ the above implies:
$$-G(\mW_{0}) || \overrightarrow{\mW}^{*} || +2k \eta v \alpha N_{e}(1-e^{-\varepsilon_{0}}) \leq || \overrightarrow{\mW}^{*} || G(\mW_{0}) + || \overrightarrow{\mW}^{*} || \sqrt{T} \sqrt{2k\eta^{2}v^{2}R_{x}^2 +2\eta}$$

Now using $\left|\left| \vw_{0}^{(i)} \right|\right|, \left|\left| \vu_{0}^{(i)} \right|\right| \leq R_{0}$ we get $G(\mW_{0}) \leq \sqrt{2k}R_{0}$.

Noting that $\left\| \overrightarrow{\mW}^{*} \right\| =\sqrt{2k} || \vw^{*} || $ and that $N_{e}=\frac{T}{n}$, we get :

\begin{equation*}
\displaystyle
    \left(\frac{2k\eta v\alpha(1-e^{-\varepsilon_{0}})}{n}\right)T \leq \sqrt{4k^2\eta^2v^2R_{x}^2 +4k\eta} || \vw^{*} || \sqrt{T}+4kR_{0} || \vw^{*} ||
\end{equation*}
Therefore, we have an inequality of the form:
$$aT \leq b\sqrt{T}+c$$
where $\displaystyle a=\frac{2k\eta v\alpha(1-e^{-\varepsilon_{0}})}{n},b = \sqrt{4k^2\eta^2v^2R_{x}^2 +4k\eta} || \vw^{*} ||$ and $c = 4kR_{0} || \vw^{*} ||$.

By inspecting the roots of the parabola $P(X) = x^2-\frac{b}{a}x -\frac{c}{a}$ we conclude that:
\begin{align*}
    T&\leq  \left(\frac{b}{a}\right)^2 +\sqrt{\frac{c}{a}}\frac{b}{a}+\frac{c}{a} = \frac{(4k^2\eta^2v^2R_{x}^2 +4k\eta) || \vw^{*} ||^2 n^2}{4k^2\eta^2v^2\alpha^2 (1-e^{-\varepsilon_{0}})^{2}} + \frac{\sqrt{4k^2\eta^2v^2R_{x}^2 +4k\eta } || \vw^{*} || n}{2k\eta v\alpha (1-e^{-\varepsilon_{0}})} \sqrt{\frac{4kR_{0}|| \vw^{*} || n}{2k\eta v\alpha (1-e^{-\varepsilon_{0}})}} \\ & + \frac{4kR_{0}|| \vw^{*} || n}{2k\eta v\alpha (1-e^{-\varepsilon_{0}})}= \left(\frac{R_{x}^2}{\alpha^2}+\frac{1}{k\eta v^2\alpha^2}\right) \frac{|| \vw^{*} ||^{2} n^{2}}{(1-e^{-\varepsilon_{0}})^{2}}  +\frac{\sqrt{R_{0}(8k^2\eta^2v^2R_{x}^2+8k\eta)}|| \vw^{*} ||^{1.5}n^{1.5}}{2k(\eta v\alpha)^{1.5} (1-e^{-\varepsilon_{0}})^{1.5}} \\ & +\frac{2R_{0}|| \vw^{*} || n }{\eta  v \alpha (1-e^{-\varepsilon_{0}})} 
\end{align*}

By the inequality $1-e^{-x} > \frac{x}{1+x}$ for $x > 0$ (which is equivalent to $\frac{1}{1-e^{-x}}<\frac{x+1}{x}$), with $x= \varepsilon_{0} > 0$ we get
$\frac{1}{1-e^{-\varepsilon_{0}}} < \frac{\varepsilon_{0}+1}{\varepsilon_{0}} = 1+ \frac{1}{\varepsilon_{0}}$.
Therefore for $\beta>0$ (all arguments are positive):
$$\frac{1}{(1-e^{-\varepsilon_{0}})^{\beta}} < \left(1+ \frac{1}{\varepsilon_{0}}\right)^{\beta}$$
By using the above inequality we can reach a polynomial bound on $T$:
\begin{align}
\label{Network epsilon accuracy bound}
    T &\leq  \left(\frac{R_{x}^2}{\alpha^2}+\frac{1}{k\eta v^2\alpha^2}\right) || \vw^{*} ||^{2} n^{2}\left(1+\frac{1}{\varepsilon_{0}}\right)^{2} \notag \\&+\frac{\sqrt{R_{0}(8k^2\eta^2v^2R_{x}^2+8k\eta)}|| \vw^{*} ||^{1.5}n^{1.5}(1+\frac{1}{\varepsilon_{0}})^{1.5}}{2k(\eta v\alpha)^{1.5}} +\frac{2R_{0}|| \vw^{*} || n (1+\frac{1}{\varepsilon_{0}})}{\eta v \alpha} 
\end{align}

We have shown that there is at most a finite amount of epochs $N_{e} = \frac{T}{n}$ such that there exists at least one point in each of them with a loss greater than $\varepsilon_{0}$. Therefore, there exists an epoch $1 \le i_e \le N_e+1$ such that each point sampled in the epoch has a loss smaller than $\varepsilon_{0}$. Formally, for any $(i_e-1) n + 1\leq t \leq i_e n, \ \  \ell\left(y_{t} N_{ \mW_{t-1}}(\vx_{t})\right) \leq \varepsilon_{0}$. Recall that SGD samples without replacement and therefore, each point is sampled at some $t$ in the epoch $i_e$. 

Next, we will show that there exists a time $t$ such that $L_{\sS}(\mW_{t}) <\varepsilon$ by bounding the change in the loss values during the epoch. We'll start by noticing that our loss function $\ell(x)$ is locally Lipschitz with coefficient $1$, that is because $\forall \vx \ |\ell'(\vx)| = \frac{1}{1+e^{x}} \leq 1$. With this in mind for any point $(y_{i},\vx_{i})\in \sS$ if we can bound $|y_{i}N_{\mW_{t+s}}(\vx_{i})-y_{i}N_{\mW_{t}}(\vx_{i})|$ we would also bound $|\ell\left(y_{i}N_{\mW_{t+s}}(\vx_{i})\right) - \ell\left(y_{i}N_{\mW_{t}}(\vx_{i})\right)|$.


For any iteration $(i_e-1) n + 1\leq t \leq i_e n$ and $1\leq s \leq n$ we have:

\begin{align}
    &|y_{i}N_{\mW_{t+s}}(\vx_{i})-y_{i}N_{\mW_{t}}(\vx_{i})| = |N_{\mW_{t+s}}(\vx_{i})-N_{\mW_{t}}(\vx_{i})| \notag \\&=\left|v\sum_{j=1}^{k} \left(\sigma(\vw^{(j)}_{t+s}\cdot \vx_{i})- \sigma(\vw^{(j)}_{t} \cdot \vx_{i})\right) - v\sum_{j=1}^{k} \left(\sigma(\vu^{(j)}_{t+s}\cdot \vx_{i})-\sigma(\vu^{(j)}_{t} \cdot \vx_{i})\right)\right| \notag \\&\leq v\sum_{j=1}^{k} \left|\sigma(\vw^{(j)}_{t+s}\cdot \vx_{i})- \sigma(\vw^{(j)}_{t} \cdot \vx_{i})\right| + v\sum_{j=1}^{k} \left|\sigma(\vu^{(j)}_{t+s}\cdot \vx_{i})-\sigma(\vu^{(j)}_{t} \cdot \vx_{i})\right| \notag \\&\leq 
    v\sum_{j=1}^{k} \left| \left(\vw^{(j)}_{t+s} - \vw^{(j)}_{t}\right)\cdot \vx_{i} \right| +v\sum_{j=1}^{k} \left|  \left(\vu^{(j)}_{t+s}- \vu^{(j)}_{t}\right)\cdot \vx_{i} \right| \label{eq:sigma lipschitzness} \\&\leq
    v\sum_{j=1}^{k} ||\vw^{(j)}_{t+s} - \vw^{(j)}_{t}|| \cdot ||\vx_{i}|| +v\sum_{j=1}^{k} ||\vu^{(j)}_{t+s}-\vu^{(j)}_{t}|| \cdot ||\vx_{i}|| 
    \label{eq:network difference bound lipschitz} \\&\leq vR_{x} \sum_{j=1}^{k} \left|\left|\sum_{h=1}^{s} \eta v y_{t+h}\vx_{t+h}\left|\ell'(y_{t+h}N_{\mW_{t+h-1}}(\vx_{t+h}))\right|p_{t+h-1}^{(j)} \right|\right|
    \notag \\&+vR_{x}\sum_{j=1}^{k} \left|\left|\sum_{h=1}^{s} \eta v y_{t+h}\vx_{t+h}\left|\ell'(y_{t+h}N_{\mW_{t+h-1}}(\vx_{t+h}))\right|q_{t+h-1}^{(j)} \right|\right|
    \label{eq:telescopic update rule}\\& \leq vR_{x}\sum_{j=1}^{k}\sum_{h=1}^{s} \eta v \left|\ell'(y_{t+h}N_{\mW_{t+h-1}}(\vx_{t+h}))\right| ||\vx_{t+h}|| + vR_{x}\sum_{j=1}^{k}\sum_{h=1}^{s} \eta v \left|\ell'(y_{t+h}N_{\mW_{t+h-1}}(\vx_{t+h}))\right| ||\vx_{t+h}|| 
    \notag \\& \leq
    2v^{2}R_{x}^{2}\eta k \sum_{h=1}^{s}\left|\ell'(y_{t+h}N_{\mW_{t+h-1}}(\vx_{t+h}))\right| \leq 2v^{2}R_{x}^{2} \eta k s (1-e^{-\varepsilon_{0}}) \leq 2v^{2}R_{x}^{2} \eta k n (1-e^{-\varepsilon_{0}}) \leq 2v^{2}R_{x}^{2}\eta k n \varepsilon_{0} \label{eq:loss derivative upper bound}
\end{align}

Where in \eqref{eq:sigma lipschitzness} we used the Lipschitzness of $\sigma(\cdot): \ \forall x_{1},x_{2}\in \R | \sigma(x_{1})-\sigma(x_{2}) | \leq |x_{1}-x_{2}|$, in \eqref{eq:network difference bound lipschitz} we used the Cauchy-Shwartz inequality, in \eqref{eq:telescopic update rule} we used the update rule \eqref{eq:neurons updates} recursively and finally in \eqref{eq:loss derivative upper bound} we used that if $\ell(x) \leq \varepsilon_{0}$ then $|\ell'(x)| \leq 1-e^{-\varepsilon_{0}}$ (follows from a similar derivation to \eqref{loss and loss gradient relations}) and that $1-e^{-\varepsilon_{0}} \leq \varepsilon_{0}$.

Now we can use the bound we just derived and the Lipschitzness of $\ell$ and reach

\begin{equation}
    \label{eq:end of epoch to mid epoch loss bound}
    |\ell\left(y_{i}N_{\mW_{t+s}}(\vx_{i})\right) - \ell\left(y_{i}N_{\mW_{t}}(\vx_{i})\right)| \leq 2v^{2}R_{x}^{2}\eta k n \varepsilon_{0}
\end{equation}
for any time $(i_e-1) n + 1\leq t \leq i_e n$ and $1\leq s \leq n$. We know that for all $1 \le i \le n$, there exists $ (i_e-1) n + 1 \le t_{i}^{*} \le  i_e n$ such that   $\ell(y_{i}N_{\mW_{t_{i}^{*}-1}}(\vx_{i})) \leq \varepsilon_0$. Therefore, by \eqref{eq:end of epoch to mid epoch loss bound}, for time $T^* = i_e n + 1$ and any $(y_{i},\vx_{i})\in \sS$  
we have:
\begin{equation}
    \label{eq:end of epoch loss bound}
    \ell\left(y_{i}N_{\mW_{T^{*}}}(\vx_{i})\right) \leq \ell\left(y_{i}N_{\mW_{t_{i}^{*}-1}}(\vx_{i})\right) +2v^{2}R_{x}^{2}\eta k n \varepsilon_{0} \leq \varepsilon_{0} +2v^{2}R_{x}^{2}\eta k n \varepsilon_{0}
\end{equation}

If $\forall 1\leq i \leq n \ \ell(y_{i}N_{\mW}(\vx_{i})) \leq \varepsilon$ we would get our bound $L_{\sS}(\mW) \leq \varepsilon$.

Therefore, if we set $\varepsilon_{0} = \frac{\varepsilon}{1+2v^{2}R_{x}^{2}\eta k n}$ in \eqref{eq:end of epoch loss bound} we'll reach our result. 

Setting this $\varepsilon_{0}$ at \eqref{Network epsilon accuracy bound}
leads to: 
\begin{align}
\label{Network epsilon 0 accuracy bound}
    T &\leq  \left(\frac{R_{x}^2}{\alpha^2}+\frac{1}{k\eta v^2\alpha^2}\right) || \vw^{*} ||^{2} n^{2}\left(1+\frac{1+2v^{2}R_{x}^{2}\eta k n}{\varepsilon}\right)^{2} \notag \\&+\frac{\sqrt{R_{0}(8k^2\eta^2v^2R_{x}^2+8k\eta)}|| \vw^{*} ||^{1.5}n^{1.5}\left(1+\frac{1+2v^{2}R_{x}^{2}\eta k n}{\varepsilon}\right)^{1.5}}{2k(\eta v\alpha)^{1.5}} +\frac{2R_{0}|| \vw^{*} || n \left(1+\frac{1+2v^{2}R_{x}^{2}\eta k n}{\varepsilon}\right)}{\eta v \alpha} 
\end{align}
We denote the right hand side of \eqref{Network epsilon 0 accuracy bound} plus $n$ by $M(n,\epsilon)$. \footnote{We need to add $n$ to \eqref{eq:end of epoch loss bound} because we may consider the epoch immediately after $T$.}
Note that $M(n, \epsilon) = O(\frac{n^{4}}{\varepsilon^2})$ and therefor for simplicity we can alternatively denote $M(n,\epsilon)$ to be a less tight bound of the form $\frac{Cn^{4}}{\varepsilon^{2}}$ where $C$ is a constant that depends polynomially on $R_{x},R_{0},k,\frac{1}{\alpha},\max\left\{\eta,\frac{1}{\eta}\right\},\max\left\{v,\frac{1}{v}\right\}$ and $|| \vw^{*} ||$.
Overall, we proved that after $O(\frac{n^{4}}{\varepsilon^2})$ steps, SGD will converge to a solution with $L_{\sS}(\mW_{t}) <\varepsilon$ empirical loss for some $t\le M(n,\varepsilon)$.

\newpage
\section{Proof of Theorem \ref{thm:difference between clustered leaky ReLus}}
\label{proof:difference between clustered leaky ReLus}

Before we start proving the main theorem we will prove some useful lemmas and corollaries.

We first show the following.

\begin{cor}
\label{cor:margin on difference to individual}
    if $|(\overline{\vw}-\overline{\vu}) \cdot \vx | \geq 2r ||\vx||$ then $|\overline{\vw}\cdot \vx| \geq r ||\vx|| \lor |\overline{\vu}\cdot \vx| \geq r ||\vx||$.
\end{cor}
\begin{proof}
    Assume in contradiction that $|\overline{\vw}\cdot \vx| < r ||\vx|| \land |\overline{\vu}\cdot \vx| < r ||\vx||$. then by the triangle inequality and the Cauchy-Shwartz inequality we'll get:

    $| (\overline{\vw}-\overline{\vu})\cdot \vx| \leq |\overline{\vw} \cdot \vx| +|\overline{\vu} \cdot \vx| < r ||\vx|| + r ||\vx|| = 2r ||\vx||$ in contradiction to the assumption $|(\overline{\vw}-\overline{\vu}) \cdot \vx | \geq 2r ||\vx||$.
\end{proof}

Next, we prove the following lemma, which will be used throughout the proof of the main theorem. The lemma ties the dot products with the center of the cluster to the dot products with the individual neurons:

\begin{lem}
\label{lemma:clusterization effect on neurons}
    If  $\forall 1\leq j \leq k: \ \vw^{(j)} \in Ball(\overline{\vw},r) \land \vu^{(j)} \in Ball(\overline{\vu},r)$
    then:
    $\forall \vx \in \R^{d} \ \text{s.t} \ |\overline{\vw}\cdot \vx|  \geq r||\vx|| : [\forall 1 \leq j \leq k \ \ \vw^{(j)} \cdot \vx > 0] \lor [\forall 1 \leq j \leq k \ \ \vw^{(j)} \cdot \vx < 0]$ 
    and similarly for $u$ type neurons
    $\forall \vx \in \R^{d} \ \text{s.t} \ |\overline{\vu}\cdot \vx|  \geq r||\vx|| : [\forall 1 \leq j \leq k \ \  \vu^{(j)} \cdot \vx > 0] \lor [\forall 1 \leq j \leq k \ \ \vu^{(j)} \cdot \vx < 0]$.
\end{lem}
\begin{proof}
    Let's assume that $\overline{\vw}\cdot \vx \geq r ||\vx||$, therefore $\forall 1 \leq j \leq k: \ \ \vw^{(j)} \cdot \vx = (\vw^{(j)} -\overline{\vw})\cdot \vx +\overline{\vw}\cdot \vx \geq -|| \vw^{(j)}-\overline{\vw}|| \cdot ||\vx || +r ||\vx|| > -r ||\vx|| +r ||\vx|| =0$
    where we had used Cauchy-Shwartz inequality and that $|| \vw^{(j)} -\overline{\vw}|| <r$.
    
    If $\overline{\vw}\cdot \vx \leq -r ||\vx||$, $\forall 1 \leq j \leq k: \ \ \vw^{(j)} \cdot \vx = (\vw^{(j)} -\overline{\vw})\cdot \vx +\overline{\vw}\cdot \vx <  || \vw^{(j)}-\overline{\vw}|| \cdot ||\vx || -r ||\vx||< r ||\vx|| -r ||\vx|| =0$
    the same derivation would work for $\vu$.
\end{proof}

We are now ready to move forward with proving the main lemma.

By \corref{cor:margin on difference to individual} we see that $\{\vx \in \R^{d}| \ |(\overline{\vw}-\overline{\vu})\cdot \vx| \geq 2r ||\vx||\} \subseteq \{\vx \in \R^{d}| \ |\overline{\vw}\cdot \vx|  \geq r||\vx|| \lor  |\overline{\vu}\cdot \vx| \geq r||\vx||\}$ so if we prove that:  

$\forall \vx \in \R^{d} \in  \{\vx \in \R^{d}| \ |(\overline{\vw}-\overline{\vu})\cdot \vx| \geq 2r ||\vx||\} \cap  \{\vx \in \R^{d}| \ |\overline{\vw}\cdot \vx|  \geq r||\vx|| \lor  |\overline{\vu}\cdot \vx| \geq r||\vx||\}: \ \sgn{N_{\mW}(\vx)} = \sgn{(\overline{\vw}-\overline{\vu})\cdot \vx}$ we will be done.

We'll start by showing first our lemma holds $\forall \vx \in \R^{d} \ \text{s.t} \ |\overline{\vw}\cdot \vx|  \geq r||\vx|| \land  |\overline{\vu}\cdot \vx|  \geq r||\vx|| $
and then deal with the points in which only one of the above conditions holds.
\begin{prop}
\label{thm:intersection between cones linearity}
    $\forall \vx \in \R^{d} \ \text{s.t} \ |\overline{\vw}\cdot \vx|  \geq r||\vx|| \land  |\overline{\vu}\cdot \vx|  \geq r||\vx|| :  \ \sgn{N_{\mW}(\vx)} = \sgn{(\overline{\vw}-\overline{\vu})\cdot \vx}$
\end{prop}

\begin{proof}
    Under our clusterization assumption $\forall 1\leq j \leq k: \ \vw^{(j)} \in Ball(\overline{\vw},r) \land \vu^{(j)} \in Ball(\overline{\vu},r)$ so we can use \lemref{lemma:clusterization effect on neurons} and we are left with proving that $\forall \vx \in \R^{d}$ such that for the $\vw$ neurons $ \{[\forall 1 \leq j \leq k \ \  \vw^{(j)} \cdot \vx > 0] \lor [\forall 1 \leq j \leq k \ \ \vw^{(j)} \cdot \vx < 0]\}$ and for the $\vu$ neurons $\{[\forall 1 \leq j \leq k \ \  \vu^{(j)} \cdot \vx > 0] \lor [\forall 1 \leq j \leq k \ \ \vu^{(j)} \cdot \vx < 0]\}$ we get $\sgn{N_{\mW}(\vx)} = \sgn{(\overline{\vw}-\overline{\vu})\cdot \vx}$.

    We can represent $\{\vx \in \R^{d} | \ |\overline{\vw}\cdot \vx|  \geq r||\vx|| \land  |\overline{\vu}\cdot \vx|  \geq r||\vx|| \} $ as a union of $ \{C_{+}^{+},C_{-}^{-},C_{+}^{-},C_{-}^{+}\}$ where: 
    $$C_{+}^{+} = \{\vx \in \R^{d} | \ \ \forall 1\leq j \leq k \ \ \vw^{(j)} \cdot \vx > 0  \text{  and  } \forall 1\leq j \leq k \ \ \vu^{(j)} \cdot \vx > 0\}$$
    $$C_{-}^{-} = \{\vx \in \R^{d} | \ \ \forall 1\leq j \leq k \ \ \vw^{(j)} \cdot \vx < 0  \text{  and  } \forall 1\leq j \leq k \ \ \vu^{(j)} \cdot \vx < 0\}$$
    $$C_{+}^{-} = \{\vx \in \R^{d} | \ \ \forall 1\leq j \leq k \ \ \vw^{(j)} \cdot \vx > 0  \text{  and  } \forall 1\leq j \leq k \ \ \vu^{(j)} \cdot \vx < 0\}$$
    $$C_{-}^{+} = \{\vx \in \R^{d} | \ \ \forall 1\leq j \leq k \ \ \vw^{(j)} \cdot \vx < 0  \text{  and  } \forall 1\leq j \leq k \ \ \vu^{(j)} \cdot \vx > 0\}$$
    Now we will show that $\sgn{N_{\mW}(\vx)} = \sgn{(\overline{\vw}-\overline{\vu})\cdot \vx}$ in each region, from which the claim follows.
    \begin{enumerate}
        \item If $\vx \in C_{+}^{+}$ then $N_{\mW}(\vx) = v \left(\sum\limits_{j=1}^{k} \sigma(\vw^{(j)} \cdot \vx) - \sigma(\vu^{(j)} \cdot \vx) \right)  = v\left(\sum\limits_{j=1}^{k}\vw^{(j)}-\vu^{(j)}\right)\cdot \vx$ and therefore  $\sgn{N_{\mW}(\vx)}  = \sgn{(\overline{\vw} -\overline{\vu}) \cdot \vx}$.
        \item If $\vx \in C_{-}^{-}$ then $N_{\mW}(\vx) = v \left(\sum\limits_{j=1}^{k} \sigma(\vw^{(j)} \cdot \vx) - \sigma(\vu^{(j)} \cdot \vx) \right) = \alpha v \left(\sum\limits_{j=1}^{k}\vw^{(j)}-\vu^{(j)}\right)\cdot \vx$ and therefore  $\sgn{N_{\mW}(\vx)} = \sgn{(\overline{\vw} -\overline{\vu}) \cdot \vx}$
        \item If $\vx \in C_{+}^{-}$ then both $N_{\mW}(\vx) = v\left(\sum\limits_{j=1}^{k} \sigma(\vw^{(j)} \cdot \vx) - \sigma(\vu^{(j)} \cdot \vx) \right) = v\left(\sum\limits_{j=1}^{k} \vw^{(j)} \cdot \vx - \alpha \vu^{(j)} \cdot \vx\right) >0$ and $\overline{\vw} \cdot \vx -\overline{\vu} \cdot \vx >0$. Therefore, $\sgn{N_{\mW}(\vx)} = \sgn{(\overline{\vw} -\overline{\vu}) \cdot \vx}$.
        \item If $\vx \in C_{-}^{+}$ then both $N_{\mW}(\vx) = v\left(\sum\limits_{j=1}^{k} \sigma(\vw^{(j)} \cdot \vx) - \sigma(\vu^{(j)} \cdot \vx)\right) = v \left(\sum\limits_{j=1}^{k} \alpha \vw^{(j)} \cdot \vx - \vu^{(j)} \cdot \vx\right) <0$ and $\overline{\vw} \cdot \vx -\overline{\vu} \cdot \vx <0$. Therefore, $\sgn{N_{\mW}(\vx)} = \sgn{(\overline{\vw} -\overline{\vu}) \cdot \vx}$.
    \end{enumerate}
\end{proof}

We are left with proving $\sgn{N_{\mW}(\vx)} = \sgn{(\overline{\vw}-\overline{\vu})\cdot \vx}$ holds when exactly one condition holds ,i.e., either $|\overline{\vw} \cdot \vx | \geq r ||\vx||$ or $|\overline{\vu} \cdot \vx | \geq r ||\vx||$.
\begin{prop}
    \label{thm:one cone only linearity}
    
    $$\forall \vx \in \{\vx \in \R^{d}| \ \ |\overline{\vw}\cdot \vx| < r || \vx|| \land |\overline{\vu}\cdot \vx | \geq r ||\vx|| \land |(\overline{\vw}-\overline{\vu})\cdot \vx | \geq 2r ||\vx||\}: \ \sgn{N_{\mW}(\vx)} = \sgn{(\overline{\vw}-\overline{\vu})\cdot \vx}$$
    and similarly our decision boundary is linear for points in which our condition only holds for $\overline{\vw}$:
    $$\forall \vx \in \{\vx \in \R^{d}| \ \ |\overline{\vu}\cdot \vx| < r || \vx|| \land |\overline{\vw}\cdot \vx | \geq r ||\vx||\land |(\overline{\vw}-\overline{\vu})\cdot \vx | \geq 2r ||\vx||\}: \ \sgn{N_{\mW}(\vx)} = \sgn{(\overline{\vw}-\overline{\vu})\cdot \vx}$$
\end{prop}

\begin{proof}
    We start with the domain $\{\vx \in \R^{d}| \ \ |\overline{\vw}\cdot \vx| < r || \vx|| \land |\overline{\vu}\cdot \vx | \geq r ||\vx|| \land |(\overline{\vw}-\overline{\vu})\cdot \vx | \geq 2r ||\vx||\}$
    
    i.e. our condition only holds for $\overline{\vu}$.
    
    There are two cases, and we'll prove the result for each of them:
    
    \underline{If $\overline{\vu} \cdot \vx \geq r ||\vx||$:}
    
    In this case $N_{\mW}(\vx) = v\left(\sum\limits_{j=1}^{k} \sigma(\vw^{(j)} \cdot \vx) - \sigma(\vu^{(j)} \cdot \vx)\right) = v\left(\sum\limits_{j=1}^{k} \sigma(\vw^{(j)} \cdot \vx) - k \overline{\vu} \cdot \vx\right)$.
    
    Next, for any $\vx$ in the domain, we'll denote $J_{+}^{w}(\vx) \coloneqq \{j | \vw^{(j)}\cdot \vx>0\}$ and $k_{+}^{w}(\vx) \coloneqq |J_{+}^{w}(\vx)|$ similarly $J_{-}^{w}(\vx) = \{j | \vw^{(j)}\cdot \vx<0\}$ and $k_{-}^{w}(\vx) \coloneqq |J_{-}^{w}(\vx)|$.
    Using these definitions, our network has the following form:
    
    $$N_{\mW}(\vx) = v\left(\sum\limits_{j_{+}\in J_{+}^{w}(\vx)} \vw^{(j_{+})} \cdot \vx +\alpha \sum\limits_{j_{-}\in J_{-}^{w}(\vx)} \vw^{(j_{-})} \cdot \vx -k\overline{\vu}\cdot \vx\right) =v\left(k\overline{\vw}\cdot \vx -k\overline{\vu}\cdot \vx +(\alpha-1)\sum\limits_{j_{-}\in J_{-}^{w}(\vx)} \vw^{(j_{-})} \cdot \vx\right)$$
    
    Next, we bound $\forall j \ |\vw^{(j)} \cdot \vx| =|(\vw^{(j)}-\overline{\vw}+\overline{\vw})\cdot \vx| \leq ||\vw^{(j)}-\overline{\vw}||\cdot ||\vx|| +|\overline{\vw}\cdot \vx | <2r||\vx||$
    where we used $||\vw^{(j)}-\overline{\vw}||<r$ and $|\overline{\vw}\cdot \vx| <r ||\vx||$.
    
    Now, if $(\overline{\vw}-\overline{\vu})\cdot \vx  \geq 2r ||\vx||>0$ we get that $N_{\mW}(\vx) = v\left(k(\overline{\vw}-\overline{\vu})\cdot \vx -(1-\alpha)\sum\limits_{j_{-}\in J_{-}^{w}(\vx)} \vw^{(j_{-})} \cdot \vx\right) >v\left(2r ||\vx||k - 2r ||\vx||k_{-}^{w}(\vx)(1-\alpha)\right)>0$ since $(1-\alpha)<1$ and $k_{-}^{w}(\vx) \leq k$
    and therefore $\sgn{N_{\mW}(\vx)} = \sgn{(\overline{\vw}-\overline{\vu})\cdot \vx} = 1$ for this case.
    
    If $(\overline{\vw}-\overline{\vu})\cdot \vx  \leq -2r ||\vx||<0$ we get that
    $N_{\mW}(\vx) = v\left(k(\overline{\vw}-\overline{\vu})\cdot \vx -(1-\alpha)\sum\limits_{j_{-}\in J_{-}^{w}(\vx)} \vw^{(j_{-})} \cdot \vx\right) <v\left(-2r ||\vx||k + 2r ||\vx||k_{-}^{w}(\vx)(1-\alpha)\right)<0$
    since $(1-\alpha)<1$ and $k_{-}^{w}(\vx) \leq k$. Therefore, we get that $\sgn{N_{\mW}(\vx)} = \sgn{(\overline{\vw}-\overline{\vu})\cdot \vx} = -1$ in this case.
    
    At any rate, we have shown that 
    $\forall \vx \in \{\vx \in \R^{d}| \ \ |\overline{\vw}\cdot \vx| < r || \vx|| \land \overline{\vu}\cdot \vx  \geq r ||\vx|| \land |(\overline{\vw}-\overline{\vu})\cdot \vx | \geq 2r ||\vx||\}: \ \sgn{N_{\mW}(\vx)} = \sgn{(\overline{\vw}-\overline{\vu})\cdot \vx}$.

    \vspace{5mm}
    \underline{If $\overline{\vu} \cdot \vx \leq -r ||\vx||$:}
    
    \vspace{5mm}
    First, we notice that $(\overline{\vw}-\overline{\vu})\cdot \vx > -r ||\vx|| +r ||\vx|| =0$ so $\sgn{(\overline{\vw}-\overline{\vu})\cdot \vx}=1$ again we use \lemref{lemma:clusterization effect on neurons} and from our assumption $\overline{\vu} \cdot \vx \leq -r ||\vx||$ we have  $\forall 1 \leq j \leq k \ \vu^{(j)} \cdot \vx < 0$ and we can see that our network takes the form:
    $N_{\mW}(\vx) = v\left(\sum\limits_{j=1}^{k} \sigma(\vw^{(j)} \cdot \vx) - \sigma(\vu^{(j)} \cdot \vx)\right) = v\left(\sum\limits_{j=1}^{k} \sigma(\vw^{(j)} \cdot \vx) - \alpha \cdot k \overline{\vu} \cdot \vx\right) \geq v\left(\sum\limits_{j=1}^{k} \sigma(\vw^{(j)} \cdot \vx) + \alpha k r ||\vx||\right)$. Next, we prove the following lemma:
    
    \begin{lem}
        \label{lem:alpha greater than sum}
        If $|\overline{\vw} \cdot \vx | < r ||\vx||$ then $\alpha \cdot k \cdot r ||\vx|| > -\sum\limits_{j=1}^{k} \sigma(\vw^{(j)} \cdot \vx)$.
    \end{lem}
    \begin{proof}
        Let's assume by contradiction that $-\sum\limits_{j=1}^{k} \sigma(\vw^{(j)} \cdot \vx)  \geq \alpha \cdot k \cdot r ||\vx||$.
        We notice that regardless of the sign of the dot product $\forall j: -\sigma(\vw^{(j)} \cdot \vx) \leq -\alpha \vw^{(j)} \cdot \vx$ so we have
        $-\alpha \sum\limits_{j=1}^{k} \vw^{(j)} \cdot \vx \geq -\sum\limits_{j=1}^{k} \sigma(\vw^{(j)} \cdot \vx) \geq \alpha \cdot k \cdot r ||\vx||$,
        which leads to $-\alpha k \overline{\vw} \cdot \vx \geq \alpha \cdot k \cdot r ||\vx||$ (where we used the definition of $\overline{\vw}$)
        finally we reach $\overline{\vw} \cdot \vx \leq -r ||\vx||$. This contradicts $|\overline{\vw} \cdot \vx | < r ||\vx||$.
    \end{proof}
    Therefore, we have $-\sum\limits_{j=1}^{k} \sigma(\vw^{(j)} \cdot \vx) < \alpha \cdot k \cdot r ||\vx||$ and $\sgn{N_{\mW}(\vx)} = \sgn{(\overline{\vw}-\overline{\vu})\cdot \vx)} = 1$ as desired.
    
    To conclude we proved that $\forall \vx \in \{\vx \in \R^{d}| \ \ |\overline{\vw}\cdot \vx| < r || \vx|| \land |\overline{\vu}\cdot \vx |  \geq r ||\vx|| \land |(\overline{\vw}-\overline{\vu})\cdot \vx | \geq 2r ||\vx||\}, \ \sgn{N_{\mW}(\vx)} = \sgn{(\overline{\vw}-\overline{\vu})\cdot \vx}$.
    
    \vspace{10mm}
    Next we look at $\forall \vx \in \{\vx \in \R^{d}| \ \ |\overline{\vu}\cdot \vx| < r || \vx|| \land |\overline{\vw}\cdot \vx |  \geq r ||\vx|| \land |(\overline{\vw}-\overline{\vu})\cdot \vx | \geq 2r ||\vx||\}$
    and through a similar derivation of two cases we will prove that $\sgn{N_{\mW}(\vx)} = \sgn{(\overline{\vw}-\overline{\vu})\cdot \vx}$.

    \vspace{5mm}
    \underline{If $\overline{\vw} \cdot \vx \geq r ||\vx||$:}
    
    \vspace{5mm}
    Through a similar derivation for the case of $\overline{\vu} \cdot \vx \geq r ||\vx||$, our network has the following form:
    \begin{align*}
        N_{\mW}(\vx) &= v\left(\sum\limits_{j=1}^{k} \sigma(\vw^{(j)} \cdot \vx) - \sigma(\vu^{(j)} \cdot \vx)\right) = v\left(k\overline{\vw}\cdot \vx -\sum\limits_{j=1}^{k} \sigma(\vu^{(j)}\cdot \vx) \right) \\&= vk\overline{\vw}\cdot \vx - v\left(\sum\limits_{j_{+}\in J_{+}^{u}(\vx)} \vu^{(j_{+})} \cdot \vx +\sum\limits_{j_{-}\in J_{-}^{u}(\vx)} \alpha \vu^{(j_{-})} \cdot \vx \right) \\&= vk\overline{\vw}\cdot \vx - v\left(\sum\limits_{j_{+}\in J_{+}^{u}(\vx)} \vu^{(j_{+})} \cdot \vx +\sum\limits_{j_{-}\in J_{-}^{u}(\vx)} \vu^{(j_{-})} \cdot \vx +(\alpha-1)\sum\limits_{j_{-}\in J_{-}^{u}(\vx)} \vu^{(j_{-})} \cdot \vx \right) \\&= v\left(k\overline{\vw}\cdot \vx -k\overline{\vu}\cdot \vx +(1-\alpha)\sum\limits_{j_{-} \in J_{-}^{u}(\vx)} \vu^{(j_{-})}\cdot \vx\right)   
    \end{align*}
    where $J_{-}^{u}(\vx) \coloneqq \{j| \vu^{(j)}\cdot\vx<0\}, J_{+}^{u}(\vx) \coloneqq \{j| \vu^{(j)} \cdot \vx >0\}$ and $k_{-}^{u}(\vx) = |J_{-}^{u}(\vx)|, k_{+}^{u}(\vx) = |J_{+}^{u}(\vx)|$.
    
    If $(\overline{\vw}-\overline{\vu})\cdot \vx \geq 2r ||\vx||>0$ then $N_{\mW}(\vx) = v\left(k(\overline{\vw}-\overline{\vu})\cdot \vx +(1-\alpha)\sum\limits_{j_{-} \in J_{-}^{u}(\vx)} \vu^{(j_{-})}\cdot \vx\right) \geq v\left(2kr ||\vx|| -2r ||\vx||(1-\alpha)k_{-}^{u}(\vx)\right)>0$ (because $(1-\alpha)<1$ and $k_{-}^{u}(\vx)\leq k$)
    and $\sgn{N_{\mW}(\vx)} = \sgn{(\overline{\vw}-\overline{\vu})\cdot \vx} = 1$ (where we used the fact that $\forall j : |\vu^{(j)}\cdot \vx |<2r ||\vx||$ which follows from $|\overline{\vu} \cdot \vx | <r ||\vx||$ and $||\vu^{(j)}-\overline{\vu}|| <r$).

    If $(\overline{\vw}-\overline{\vu})\cdot \vx \leq -2r ||\vx||<0$ we get that $N_{\mW}(\vx) \leq v\left(-2r||\vx||k+2r ||\vx|| (1-\alpha)k_{-}^{u}(\vx)\right)<0$ and $\sgn{N_{\mW}(\vx)} = \sgn{(\overline{\vw}-\overline{\vu})\cdot \vx} = 1$.
    
    To summarize, we showed that $\forall \vx \in \{\vx \in \R^{d}| \ \ |\overline{\vu}\cdot \vx| < r || \vx|| \land \overline{\vw}\cdot \vx   \geq r ||\vx|| \land |(\overline{\vw}-\overline{\vu})\cdot \vx | \geq 2r ||\vx||\} $, $\sgn{N_{\mW}(\vx)} = \sgn{(\overline{\vw}-\overline{\vu})\cdot \vx}$.
    
    \vspace{5mm}
    \underline{If $\overline{\vw} \cdot \vx \leq -r ||\vx||$:}
    
    \vspace{5mm}
    We again use \lemref{lemma:clusterization effect on neurons} which yields from $\overline{\vw} \cdot \vx \leq -r ||\vx||$ that  $\forall 1 \leq j \leq k \ \vw^{(j)} \cdot \vx < 0$ and we can see that our network takes the form:
    
    $$N_{\mW}(\vx) = v\left(\sum\limits_{j=1}^{k} \sigma(\vw^{(j)} \cdot \vx) - \sigma(\vu^{(j)} \cdot \vx)\right) =  \alpha k v \overline{\vw} \cdot \vx  - v\left(\sum\limits_{j=1}^{k} \sigma(\vu^{(j)} \cdot \vx)\right) \leq v\left(-\alpha k  r ||\vx|| - \sum\limits_{j=1}^{k} \sigma(\vu^{(j)} \cdot \vx)\right)$$
    If $-\sum\limits_{j=1}^{k} \sigma(\vu^{(j)} \cdot \vx) < \alpha k r ||\vx||$ we have $\sgn{N_{\mW}(\vx)} = \sgn{(\overline{\vw}-\overline{\vu})\cdot \vx} = -1$ as desired.
    
    The same contradiction proof from $\overline{\vu} \cdot \vx \leq -r ||\vx||$ segment above (\lemref{lem:alpha greater than sum}) would show
    
    $-\sum\limits_{j=1}^{k} \sigma(\vu^{(j)} \cdot \vx) < \alpha \cdot k \cdot r ||\vx||$ (just exchange $w$ and $u$) and we'll get $\sgn{N_{\mW}(\vx)} = \sgn{(\overline{\vw}-\overline{\vu})\cdot \vx} = -1$.
    
    \vspace{10 mm}
    
    Finally, we proved that 
    $$\forall \vx \in \{\vx \in \R^{d}| \ \ |\overline{\vw}\cdot \vx| < r || \vx|| \land |\overline{\vu}\cdot \vx | \geq r ||\vx|| \land |(\overline{\vw}-\overline{\vu})\cdot \vx | \geq 2r ||\vx||\} \ \sgn{N_{\mW}(\vx)} = \sgn{(\overline{\vw}-\overline{\vu})\cdot \vx}$$
    and that
    $$\forall \vx \in \{\vx \in \R^{d}| \ \ |\overline{\vu}\cdot \vx| < r || \vx|| \land |\overline{\vw}\cdot \vx | \geq r ||\vx||\land |(\overline{\vw}-\overline{\vu})\cdot \vx | \geq 2r ||\vx||\} \ \sgn{N_{\mW}(\vx)} = \sgn{(\overline{\vw}-\overline{\vu})\cdot \vx}$$ as required.
\end{proof}

We can now combine \corref{cor:margin on difference to individual}, Proposition \ref{thm:intersection between cones linearity} and Proposition \ref{thm:one cone only linearity} and prove \thmref{thm:difference between clustered leaky ReLus}:

We have $\forall \vx \in \R^{d} \text{ s.t }|(\overline{\vw}-\overline{\vu}) \cdot \vx | \geq 2r ||\vx||$ then $|\overline{\vw}\cdot \vx| \geq r ||\vx|| \lor |\overline{\vu}\cdot \vx| \geq r ||\vx||$. If $\vx$ is such that $|\overline{\vw}\cdot \vx | \geq r ||\vx|| \land |\overline{\vu}\cdot \vx | \geq r ||\vx||$ we can use Proposition (\ref{thm:intersection between cones linearity}) and get $\sgn{N_{\mW}(\vx)} = \sgn{(\overline{\vw}-\overline{\vu})\cdot \vx}$.

If only one condition holds i.e. $\vx \in \{\vx \in \R^{d}| \ \ |\overline{\vw}\cdot \vx| < r || \vx|| \land |\overline{\vu}\cdot \vx | \geq r ||\vx|| \land |(\overline{\vw}-\overline{\vu})\cdot \vx | \geq 2r ||\vx||\}$ or $\vx \in \{\vx \in \R^{d}| \ \ |\overline{\vu}\cdot \vx| < r || \vx|| \land |\overline{\vw}\cdot \vx | \geq r ||\vx||\land |(\overline{\vw}-\overline{\vu})\cdot \vx | \geq 2r ||\vx||\}$ then we can use Proposition (\ref{thm:one cone only linearity}) and get $\sgn{N_{\mW}(\vx)} = \sgn{(\overline{\vw}-\overline{\vu})\cdot \vx}$.

Therefore, overall for $|(\overline{\vw}-\overline{\vu})\cdot \vx| \geq 2r||\vx||$ we get $\sgn{N_{\mW}(\vx)} = \sgn{(\overline{\vw}-\overline{\vu})\cdot \vx}$ as required.
\subsection{Proof of Corollary \ref{cor:perfectly_clustered_linear}}

Since the network is perfectly clustered, the corollary follows by Proposition (\ref{thm:intersection between cones linearity}) with $r=0$.


\section{Additional Experiments - Linear Decision Boundary}
In this section we provide additional empirical evaluations of the decision boundary that SGD converges to in our setting.
\subsection{Leaky ReLU vs ReLU decision boundary}
\thmref{thm:difference between clustered leaky ReLus} addresses the case of Leaky ReLU activation. Here we show that the result is indeed not true for ReLU networks. We compare two perfectly clustered networks (i.e., each with two neurons) one with a Leaky ReLU activation and the other with a ReLU activation. \figref{fig:ReLU vs Leaky ReLU comparison} shows a decision boundary for a two neuron network, in the case of Leaky ReLU (\figref{fig:Leaky ReLU network - Linear Decision Boundary}) and ReLU (\figref{fig:ReLU network - Non Linear Decision Boundary}). It can be seen that the leaky ReLU indeed provides a linear decision boundary, as predicted by Theorem \ref{thm:difference between clustered leaky ReLus}, whereas the ReLU case is non-linear (we explicitly show the regime where the network output is zero. This can be orange or blue, depending on whether zero is given label positive or negative. In any case the resulting boundary is non-linear).
\begin{figure*}[h!]
    \centering
    \begin{subfigure}[b]{0.475\textwidth}
        \centering
        \includegraphics[scale=0.45]{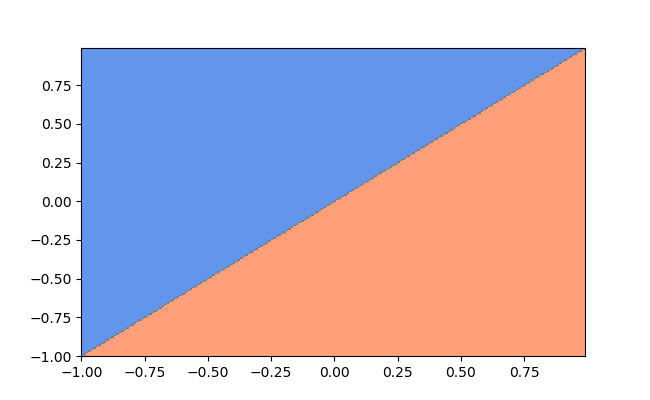}
        \caption[]%
        {Leaky ReLU network - Linear Decision Boundary}    
        \label{fig:Leaky ReLU network - Linear Decision Boundary}
    \end{subfigure}
    \hfill
    \begin{subfigure}[b]{0.475\textwidth}  
        \centering 
        \includegraphics[scale=0.45]{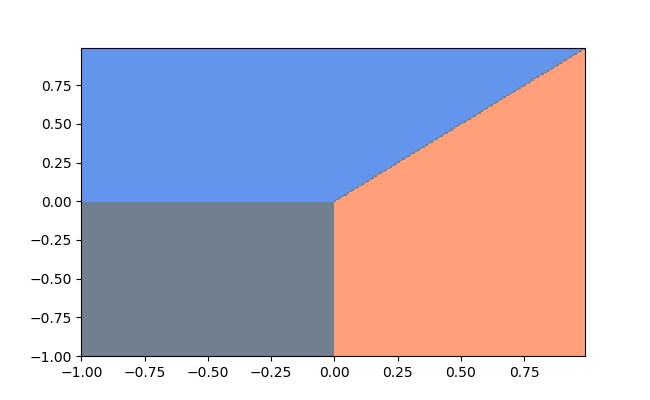}
        \caption[]%
        {ReLU network - Non Linear Decision Boundary}    
        \label{fig:ReLU network - Non Linear Decision Boundary}
    \end{subfigure}
    \caption[ ]{
    The prediction landscape for two neuron networks with Leaky ReLU and ReLU activations. Orange for positive prediction, blue for a negative prediction and grey for zero prediction. The $\vw$ neuron is $(1,0) \in \R^{2}$ and the $\vu$ neuron is $(0,1) \in \R^{2}$.}
    \label{fig:ReLU vs Leaky ReLU comparison}
\end{figure*}


\subsection{MNIST - Linear Regime}
In \figref{fig:Clusterization fig} in the main text we saw how for MNIST digit pairs (0,1) and (3,5) the network enters the linear regime at some point in the training process. In \figref{fig:MNIST linear regime - in depth analysis} we see the robustness of this behavior across the MNIST data-set by showing the above holds for more pairs of digits.
\begin{figure}[h!]
    \centering
        \includegraphics[scale=0.5]{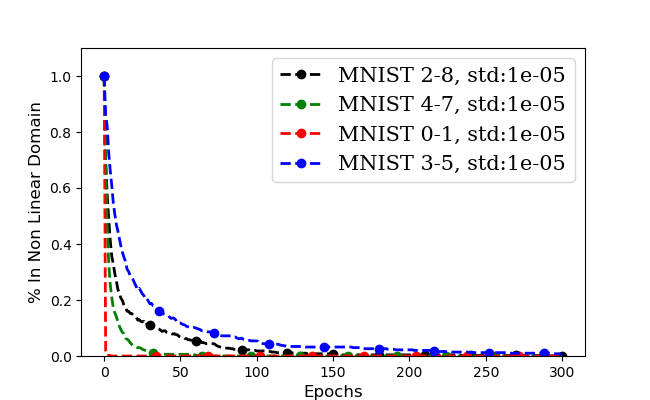}
    \caption[ ]{Convergence to a classifier that is linear on the data, for MNIST pairs. Each line corresponds to an average over 5 initializations.}
    \label{fig:MNIST linear regime - in depth analysis}
\end{figure}

\subsection{Clustering of Neurons - Empirical Evidence}
In Section \ref{sec:quasi_linear} in the main text and \figref{fig:MNIST linear regime - in depth analysis} above, we saw that learning converges to a linear decision boundary on the train and test points. \thmref{thm:difference between clustered leaky ReLus} suggests that this will happen if neurons are well clustered (in the $\vw$ and $\vu$ groups). Here we show that indeed clustering occurs.

We consider two different measures of clustering. The first is the ratio $\frac{r}{||\overline{\vw}-\overline{\vu}||}$, and the second is the maximum angle between the neurons of the same type (i.e., the maximal angle between vectors in the same cluster). \figref{fig:Neural Clusterization} shows these two measures as a function of the training epochs. They can indeed be seen to converge to zero, which by \thmref{thm:difference between clustered leaky ReLus} implies convergence to a linear decision boundary.

\begin{figure}[h!]
    \centering
    \begin{subfigure}[b]{0.475\textwidth}
        \includegraphics[scale=0.5]{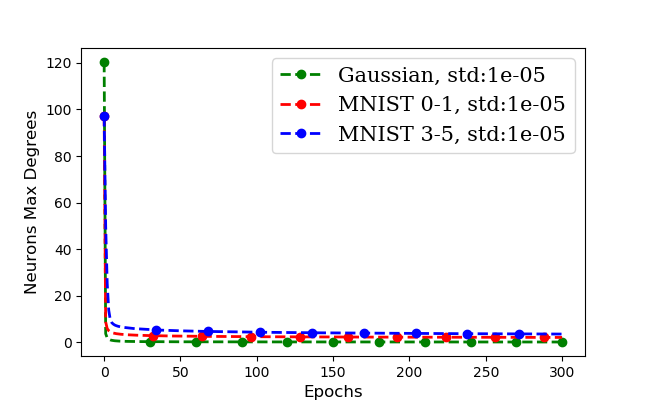}
        \caption[]%
        {Max Angle In Same Cluster Neurons} 
        \label{fig:Neuron Clusterization - Weights Angle}
    \end{subfigure}
    \begin{subfigure}[b]{0.475\textwidth}
        \includegraphics[scale=0.5]{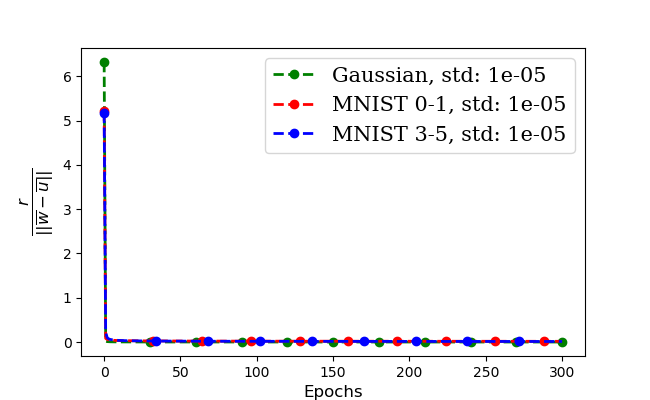}
        \caption[]%
        {$\frac{r}{||\overline{\vw}-\overline{\vu}||}$}    
        \label{fig:Neuron Clusterization - radius to norm}
    \end{subfigure}
    \caption[ ]{Evaluation of clustering measures during training. We consider two different clustering measures in (a) and (b) (see text). 
    It can be seen that both measures converge to zero.}
    \label{fig:Neural Clusterization}
\end{figure}

\section{Assumptions for Gradient Flow Analysis}
\label{sec:PAPERS ASSUMPTIONS}
In the paper we use results from \cite{lyu2019gradient} and \cite{ji2020directional}. Here we show that the assumptions required by these theorems are satisfied in our setup.

The assumptions in \cite{lyu2019gradient} and \cite{ji2020directional} are:
\begin{itemize}
    \item[(\textbf{A1})]. (Regularity). For any fixed $\displaystyle \vx, \Phi(\cdot ;\vx)$ is locally Lipschitz and admits a chain rule;
    \item[(\textbf{A2})]. (Homogeneity). There exists $L>0$ such that $\displaystyle \forall \alpha >0 : \Phi(\alpha \mW ;\vx) = \alpha^{L}\Phi(\mW;\vx);$
    \item[(\textbf{B3})]. The loss function $\ell(q)$ can be expressed as $\ell(q) = e^{-f(q)}$ such that
    \begin{itemize}
        \item[(B3.1).] $f : \R \rightarrow \R$ is $\displaystyle \mathcal{C}^{1}$-smooth.
        \item[(B3.2).] $f'(q)>0$ for all $\displaystyle q \in \R$.
        \item[(B3.3).] There exists $b_{f} \geq 0$ such that $\displaystyle f'(q)q$ is non-decreasing for $q \in (b_{f},+\infty)$, and $f'(q)q \rightarrow +\infty$ as $q \rightarrow + \infty$.
        \item [(B3.4).] Let $g : [f(b_{f}),+\infty) \rightarrow [b_{f},+\infty)$ be the inverse function of $f$ on the domain $[b_{f},+\infty).$ There exists $b_{g} \geq max\{2f(b_f),f(2b_{f})\}, K\geq 1$ such that $\displaystyle g'(x) \leq Kg'(\theta x)$ and $f'(y) \leq Kf'(\theta y)$ for all $x \in (b_{g},+\infty), y \in (g(b_{g}),+\infty)$ and $\theta \in [1/2,1)$
    \end{itemize}
    \item[(\textbf{B4}).] (Separability). There exists a time $t_{0}$ such that $\mathcal{L}(\mW) <e^{-f(b_{f})} = \ell(b_{f})$
\end{itemize}
We next show that these are satisfied in our setup.
\begin{proof}
\begin{itemize}
\label{assumptions:ICLR paper assumptions proof}
    \item[(\textbf{A1}).] (Regularity) first we show that $\Phi(\cdot;\vx)$ is locally Lipschitz, with slight abuse of notations, let $\mW_{1} =\overrightarrow{\mW}_{1}, \mW_{2} = \overrightarrow{\mW}_{2} \in \R^{2kd}$ so in our case:
    \begin{align*}
        &\Phi(\mW_{1};\vx) - \Phi(\mW_{2};\vx) = \vv \cdot \sigma(\mW_{1}\cdot \vx) - \vv \cdot \sigma(\mW_{2} \cdot \vx) \\ &=  v \left[\sum\limits_{j=1}^{k} \sigma\left(\vw_{1}^{(j)} \cdot \vx\right) - \sigma\left(\vw_{2}^{(j)} \cdot \vx\right) - \left(\sigma\left(\vu_{1}^{(j)} \cdot \vx \right) - \sigma\left(\vu_{2}^{(j)} \cdot \vx \right)\right)\right]
    \end{align*}
    and therefore
    \begin{align*}
        &||\Phi(\mW_{1};\vx)- \Phi(\mW_{2};\vx)|| \\& =  \left| \left| v \left[\sum\limits_{j=1}^{k} \sigma\left(\vw_{1}^{(j)} \cdot \vx \right) - \sigma\left(\vw_{2}^{(j)}\cdot \vx \right) - \left(\sigma\left(\vu_{1}^{(j)} \cdot \vx\right) - \sigma\left(\vu_{2}^{(j)} \cdot \vx\right)\right)\right]\right| \right|  \\  &\leq v\left[\sum\limits_{j=1}^{k} \left|\left|\sigma\left(\vw_{1}^{(j)} \cdot \vx \right) - \sigma\left(\vw_{2}^{(j)} \cdot \vx \right)\right|\right| +  \left|\left|\sigma\left(\vu_{1}^{(j)} \cdot \vx \right) - \sigma\left(\vu_{2}^{(j)} \cdot \vx \right) \right|\right| \right]  \\ &\leq 2v ||\vx||\left[\sum\limits_{j=1}^{k} ||\vw_{1}^{(j)} - \vw_{2}^{(j)}|| + ||\vu_{1}^{(j)}-\vu_{2}^{(j)}||\right] =2v \cdot||\vx||\cdot ||\overrightarrow{\mW}_{1}-\overrightarrow{\mW}_{2}||
    \end{align*}
     And we showed $\Phi(\cdot;\vx)$ is globally Lipschitz (and therfor locally Lispchitz). Next for the chain rule,
     as shown in \cite{davis2018stochastic}  (corollary for deep learning therein), any function definable in an o-minimal structure admits a chain rule. Our network is definable because algebraic, composition, inverse, maximum and minimum operations over definable functions are also definable. Leaky ReLUs are definable as maximum operations over two linear functions (linear functions are definable).and because Leaky ReLUs are definable our network is also definable.
    \item[(\textbf{A2}).] (Homogeneity). It is easy to see from the definition that in our case, the trainable parameters are only the first layer weights and the network $\displaystyle \Phi(\cdot ; \vx)$ is $L=1$ homogeneous.
    \item[(\textbf{B3}).] As seen in \citet{lyu2019gradient} ({Remark A.2.} therein) the logistic loss $\displaystyle \ell(q) = \log(1+e^{-q})$ satisfies $\displaystyle (\textbf{B3})$ with $\displaystyle f(q) = -\log\left(\log(1+e^{-q})\right), g(q) = -\log\left(e^{e^{-q}}-1\right),b_f=0$.
    \item[(\textbf{B4}).] (Separability). This is Assumption \ref{assump:late_phase} in the main text. As we mentioned in the main text, this assumption is satisfied with SGD by \thmref{thm:Risk Convergence}.
\end{itemize}
\end{proof}

\section{Proof of Theorem \ref{thm:Neural Alignment}} 
\label{sec:proof NAR to alignment}
In this proof we will show that the normalized parameters $\hat{\mW}_{t} \coloneqq \frac{\mW_{t}}{|| \mW_{t} ||}$ under gradient flow optimization, converges to a solution in $\gN$ and that the network $N_{\hat{\mW}}$ at convergence is perfectly clustered.
Under our assumption $\forall t \geq T_{NAR}$ $\hat{\mW}_{t} \in \gN$. From the definition of the NAR it's easy to see that the NAR is a closed domain. Therefore any limit point of $\hat{\mW}_{t}$ is also in the NAR.
From  \citet{ji2020directional} (Theorem 3.1. therein) we have that the normalized parameters flow converges when using gradient flow.
To conclude so far, we had shown that $\hat{\mW}_{t}$ converges to a point inside the NAR $\gN$.

We are left with showing that the limit point of $\underset{t \rightarrow \infty}{\lim} \hat{\mW}_{t} \coloneqq \hat{\mW}_{*}$ has a perfectly clustered form.

\citet{lyu2019gradient} (Theorem A.8. therein) shows that every limit point of $\hat{\mW}_{t}$ is along the direction of a KKT point of the following optimization problem (P):
\begin{align*}
    \displaystyle &  \min \frac{1}{2} || \mW ||_{2}^{2}
  \\
     \text{s.t.}  \quad & q_{i}(\mW) \geq 1 \quad \quad \forall i \in [n]
\end{align*}

where $q_{i}(\mW) = y_{i}N_{\mW}(\vx_{i})$ is the network margin on the sample point $(y_{i},\vx_{i})$.\footnote{It is not hard to see that given that the solution is in an NAR, then this optimization problem is convex.}

We are left with showing that at convergence the neurons align in two directions. We will use a characterization of the KKT points of  \hyperref[KKT problem]{(P)} and show that they are perfectly clustered. Since every limit point of the normalized parameters flow is along the direction of a KKT point of \hyperref[KKT problem]{(P)} that would mean $\hat{\mW}_{*}$ has a perfectly clustered form.

A feasible point $\mW$ of (P) is a KKT point if there exist $\lambda_{1},\dots, \lambda_n \geq 0$ such that:
\begin{enumerate}
    \item $\mW - \sum\limits_{i=1}^{n} \lambda_{i}\vh_{i} = 0 $ for some $\vh_{1},\dots,\vh_{n}$ satisfying $\vh_{i} \in \partial^{\circ}q_{i}(\mW)$
    \item $\forall i \in [n]: \lambda_{i}(q_{i}(\mW)-1)=0$
\end{enumerate}

From \citet{lyu2019gradient} (Theorem A.8. therein) we know $\exists \beta$ s.t. $\beta \hat{\mW}_{*}$ is a KKT point of \hyperref[KKT problem]{(P)}. Since our limit point is in an NAR we don't need to worry about the non differential points of the network because $\forall 1 \leq j \leq k,i\in[n]: \vw^{(j)}_{*} \cdot \vx_{i} \not = 0 \land \vu^{(j)}_{*} \cdot \vx_{i} \not=0$. (where $\vw^{(j)}_{*}$ and $\vu^{(j)}_{*}$ stands for the $\vw$ and $\vu$ type neurons of $\mW_{*}$, respectively). Therefore the Clarke subdifferential coincides with the gradient in our domain, and we can derive it using calculus rules.


By looking at the gradient of the margin for any point $(y_{i},\vx_{i})$:
\begin{itemize}
    \item $\displaystyle \frac{\partial q_{i}(\mW)}{\partial \vw^{(j)}} =  \frac{y_{i}\partial N_{\mW}(\vx_{i})}{\partial \vw^{(j)}} = y_{i}        v\vx_{i}\sigma'({\vw^{(j)}} \cdot \vx_{i}) = y_{i}v \vx_{i}\sigma'({\vw^{(j)}} \cdot \vx_{i})$
    \item $\displaystyle \frac{\partial q_{i}(\mW)}{\partial \vu^{(j)}} =  \frac{y_{i}\partial N_{\mW}(\vx_{i})}{\partial \vu^{(j)}} = -y_{i} v \vx_{i}\sigma'({\vu^{(j)}} \cdot \vx_{i}) = -y_{i}v \vx_{i}\sigma'({\vu^{(j)}} \cdot \vx_{i})$    
\end{itemize}

Now using the above gradients implies that:
$\partial q_{i}(\mW) = y_{i} v \vx_{i}(\overbrace{\sigma'({\vw^{(1)}} \cdot \vx_{i}),\dots,\sigma'({\vw^{(k)}} \cdot \vx_{i})}^{k},\overbrace{-\sigma'({\vu^{(1)}} \cdot \vx_{i}),\dots,-\sigma'({\vu^{(k)}} \cdot \vx_{i})}^{k})$

By the definition of the NAR $\mathcal{N}$ with parameters $(\beta,c^{\vw}_{i},c^{\vu}_{i})$ the dot product of a point $\vx_{i}$ with all neurons of the same type is of the same sign, i.e.:
$$\forall i \in [n] , \forall 1\leq l,p \leq k: \sigma'({\vw^{(l)}} \cdot \vx_{i}) = \sigma'({\vw^{(p)}} \cdot \vx_{i}) = c^{\vw}_{i}$$
and
$$\forall i \in [n] , \forall 1\leq l,p \leq k: \sigma'({\vu^{(l)}} \cdot \vx_{i}) = \sigma'({\vu^{(p)}} \cdot \vx_{i}) = c^{\vu}_{i}$$

It follows that for $\mW \in \mathcal{N}$, $\partial q_{i}(\mW) = y_{i}\cdot v\cdot \vx_{i}(\overbrace{c^{\vw}_{i},\dots,c^{\vw}_{i}}^{k},\overbrace{-c^{\vu}_{i},\dots,-c^{\vu}_{i}}^{k})$.

Therefore, by the definition of a KKT point we have:
\begin{equation*}
    \hat{\mW}_{*} =  \frac{1}{\beta}\left(\overbrace{ \sum\limits_{i=1}^{n} \lambda_{i} y_{i} v \vx_{i} c^{\vw}_{i},\dots,\sum\limits_{i=1}^{n} \lambda_{i} y_{i}\ v \vx_{i} c^{\vw}_{i}}^{k},\overbrace{-\sum\limits_{i=1}^{n} \lambda_{i} y_{i} v \vx_{i} c^{\vu}_{i},\dots,-\sum\limits_{i=1}^{n} \lambda_{i} y_{i} v \vx_{i} c^{\vu}_{i}}^{k}\right) \in \R^{2kd}
\end{equation*}

We can see that the first $k$ entries are equal, as well as the next $k$ entries (equal to each other and not to the first $k$ entries).

Therefore the normalized parameters flow $\hat{\mW}_{t}$ converges to a perfectly clustered solution.
\subsection{Proof Of Corollary 6.2.}
 By \thmref{thm:Neural Alignment}, we know the normalized parameters $\hat{\mW_{t}}$ are perfectly clustered at convergence so by \corref{cor:perfectly_clustered_linear} we get that the decision boundary of $N_{\hat{\mW}}$ is linear at convergence. From the homogeneity of the network we have $N_{\mW}(\vx)= ||\mW||N_{\hat{\mW}}(\vx)$ for any $\mW \in \R^{2kd}$ and because the norm is a non negative scalar we get $\sgn{N_{\mW}(\vx)} = \sgn{N_{\hat{\mW}}(\vx)}$, i.e., $N_{\mW}$ and $N_{\hat{\mW}}$ are the same classifiers. Therefore, this implies that the decision boundary of $N_{\mW}$ is linear at convergence.\footnote{We use $\sgn\infty = 1$ and $\sgn{-\infty} = -1$, since the norm $||\mW||$ diverges.}

\section{Proof of Theorem \ref{thm:multi neuron PAR}}
We divide the proof of \thmref{thm:multi neuron PAR} into two parts. First, we show that the NAR is a PAR, and then we show that if a network enters and remains in the PAR the network weights at convergence are proportional to the solutions of the SVM problem we defined in the main text.
\subsection{The NAR is a PAR}
\label{Appendix:proof of NAR to PAR alignment}



In this subsection we will prove the NAR is in fact a PAR under the conditions of the theorem. In the first step we show that for all $\vw^{(i)}$'s, $\left(\frac{{\vw^{(i)}}}{||\vw^{(i)}||}\right) \cdot \vx_{+} \geq \beta$ for all positive $\vx_{+} \in \sS_{+}$ and times $t \geq T_{Margin}$.
Assume by contradiction that the latter does not hold. Thus, by assumption 2 the network is in a NAR$(\beta)$ and there exists a positive $\displaystyle \vx_{+} \in \sS_{+}$ such that $\left(\frac{{\vw^{(i)}}}{||\vw^{(i)}||}\right) \cdot \vx_{+} \leq -\beta$ for all $\displaystyle \vw^{(i)}$. Denote by $\overline{\gamma}_{t,\{\vx\}}$ the margin of the network at time $t \geq T_{Margin}$ on the point $\vx$. we notice that $\overline{\gamma}_{t}\leq \overline{\gamma}_{t,\{\vx\}}$ by definition. Then:
\begin{align}
    \tilde{\gamma}_{t} &\leq \overline{\gamma}_{t} \leq \overline{\gamma}_{t,\{\vx_{+}\}} = \frac{+1\cdot N_{\mW}(\vx_{+})}{\left\| \overrightarrow{\mW} \right\|} = \frac{v\left(\sum\limits_{i=1}^{k} \sigma\left({\vw^{(i)}_{t}} \cdot \vx_{+}\right) - \sum\limits_{i=1}^{k} \sigma\left({\vu^{(i)}_{t}} \cdot\vx_{+}\right)\right)}{\sqrt{\sum\limits_{i=1}^{k} \left|\left| \vw^{(i)}_{t} \right|\right|^{2} + \left|\left| \vu^{(i)}_{t} \right|\right|^{2}}} \label{eq:smoothed margin is a margin approximation}\\& \leq \frac{v\left(\sum\limits_{i=1}^{k} \sigma\left({\vw^{(i)}_{t}} \cdot\vx_{+}\right) - \sum\limits_{i=1}^{k} \sigma\left({\vu^{(i)}_{t}} \cdot\vx_{+}\right)\right)}{\sqrt{\sum\limits_{i=1}^{k} \left|\left| \vu^{(i)}_{t} \right|\right|^{2}}} \leq
    \frac{v\left(- \sum\limits_{i=1}^{k} \sigma\left({\vu^{(i)}_{t}} \cdot\vx_{+}\right)\right)}{\sqrt{\sum\limits_{i=1}^{k} \left|\left| \vu^{(i)}_{t} \right|\right|^{2}}}   \leq \frac{v\alpha \left(\sum\limits_{i=1}^{k} \left | \vu^{(i)}_{t} \cdot\vx_{+}\right| \right)}{\sqrt{\sum\limits_{i=1}^{k} \left|\left| \vu^{(i)}_{t} \right|\right|^{2}}} \label{eq:u type neurons dot with positive points} \\  & \leq   \frac{v\alpha \left(\sum\limits_{i=1}^{k} \left|\left|\vu^{(i)}_{t}\right|\right|\cdot ||\vx_{+}||\right)}{\sqrt{\sum\limits_{i=1}^{k} \left|\left| \vu^{(i)}_{t} \right|\right|^{2}}} \leq \frac{v\alpha \left(\sum\limits_{i=1}^{k} \left|\left|\vu^{(i)}_{t}\right|\right|\right)\cdot \underset{i \in [n]}{max}||\vx_{i}||}{\sqrt{\sum\limits_{i=1}^{k} \left|\left| \vu^{(i)}_{t} \right|\right|^{2}}} \nonumber \\ &  = \frac{v\cdot \alpha \cdot \left|\left|\left(\overbrace{\left|\left|\vu^{(1)}_{t}\right|\right|,\dots, \left|\left|\vu^{(k)}_{t}\right|\right|}^{k}\right)\right|\right|_{1}\cdot \underset{i \in [n]}{max}||\vx_{i}||}{\sqrt{\sum\limits_{i=1}^{k} \left|\left| \vu^{(i)}_{t} \right|\right|^{2}}} \nonumber
\end{align}

where the first inequality follows by \citet{lyu2019gradient} (Theorem A.7. therein). In \eqref{eq:u type neurons dot with positive points} we noticed that $- \sum\limits_{i=1}^{k} \sigma\left({\vu^{(i)}_{t}} \cdot\vx_{+}\right)$ is largest when $\forall 1 \leq i \leq k \quad {\vu^{(i)}_{t}} \cdot\vx_{+} <0$ and therefore $\sigma\left({\vu^{(i)}_{t}} \cdot\vx_{+}\right) = \alpha {\vu^{(i)}_{t}} \cdot\vx_{+}$. Therefore, by the inequality $\forall \vv \in \R^{k} \quad \displaystyle || \vv ||_{1} \leq \sqrt{k}\cdot || \vv||_{2}$, we have:

\begin{align}
    &\tilde{\gamma}_{t} \leq \frac{v\cdot \alpha \cdot \sqrt{k} \left|\left|\left(\overbrace{\left|\left|\vu^{(1)}_{t}\right|\right|,\dots, \left|\left|\vu^{(k)}_{t}\right|\right|}^{k}\right)\right|\right|_{2}\cdot \underset{i \in [n]}{max}||\vx_{i}||}{\sqrt{\sum\limits_{i=1}^{k} \left|\left| \vu^{(i)}_{t} \right|\right|^{2}}}  = \sqrt{k} \cdot \alpha \cdot  v \cdot \underset{i \in [n]}{max} || \vx_{i} || \label{eq:vector norm 1 and 2 inequality}
\end{align}

Now under assumption 3 there exists a time $T_{Margin} \geq T_{NAR}$ such that $\tilde{\gamma}_{_{T_{Margin}}} > \sqrt{k}\alpha v\cdot \max\limits_{i \in [n]} ||\vx_{i}||$. 
By \citet{lyu2019gradient} (Theorem A.7. therein) the smoothed margin $\tilde{\gamma}_{t}$ is a non-decreasing function and we will get that $\forall t \geq T_{Margin}: \ \tilde{\gamma}_{t} >\sqrt{k}\alpha v\cdot \max\limits_{i \in [n]} ||\vx_{i}||$ which is a contradiction to \eqref{eq:vector norm 1 and 2 inequality}. Hence, $\forall 1 \leq i \leq k \land \vx \in \sS_{+}: \ \left(\frac{{\vw^{(i)}_{t}} }{|| \vw^{(i)}_{t} ||}\right) \cdot \vx \geq \beta$.

In a similar fashion, assume there is some $\vx_{-} \in \sS_{-}$ such that for $\displaystyle \left(\frac{{\vu^{(l)}_{t}} }{|| \vu^{(l)}_{t} ||}\right) \cdot \vx_{-} \geq \beta$ doesn't hold. Then by assumption 1 the network is in a NAR,  $\displaystyle \left(\frac{{\vu^{(l)}_{t}} }{|| \vu^{(l)}_{t} ||}\right) \cdot \vx_{-} \leq -\beta$ and by symmetry again we get:
\begin{align*}
\displaystyle
    \tilde{\gamma}_{t} &\leq \overline{\gamma}_{t} \leq \overline{\gamma}_{t,\{\vx_{-}\}} = \frac{-1\cdot N_{\mW}(\vx_{-})}{|| \overrightarrow{\mW}_{t} ||} = \frac{-v\left(\sum\limits_{i=1}^{k} \sigma\left({\vw^{(i)}_{t}} \cdot \vx_{-}\right) - \sum\limits_{i=1}^{k} \sigma\left({\vu^{(i)}_{t}} \cdot \vx_{-}\right)\right)}{\sqrt{\sum\limits_{i=1}^{k} \left|\left| \vw^{(i)}_{t} \right|\right|^{2} + \left|\left| \vu^{(i)}_{t} \right|\right|^{2}}} \\& \dots \leq  \sqrt{k} \alpha v \cdot \underset{i \in [n]}{\max} || \vx_{i} ||
\end{align*}
By \citet{lyu2019gradient} (Theorem A.7. therein) we reach a contradiction to the network margin assumption again, so
$\\ \forall \vx \in \sS_{-}: \quad \left(\frac{{\vu^{(l)}_{t}}}{\left|\left| \vu^{(l)}_{t} \right|\right|}\right) \cdot \vx \geq \beta$.

To conclude, we have proven so far for all $t>T_{Margin}$:
\begin{enumerate}
\label{Appendix:NAR to PAR points conditions proof}

    \item $\forall 1 \leq i \leq k: ,\quad\forall \vx \in \sS_{+}: \quad \left(\frac{{\vw^{(i)}_{t}} }{|| \vw^{(i)}_{t} ||}\right) \cdot \vx \geq \beta$.
    \item $\forall  1 \leq i \leq k: ,\quad \forall \vx \in \sS_{-}: \quad \left(\frac{{\vu^{(l)}_{t}}}{\left|\left| \vu^{(l)}_{t} \right|\right|}\right) \cdot \vx \geq \beta$.
\end{enumerate}
Now, by assumption 4, $\forall \vx \in \sS_{-}: \quad \left(\frac{{\vw^{(i)}_{t}} }{|| \vw^{(i)}_{t} ||}\right) \cdot \vx \not \geq \beta$ and similarly $\forall \vx \in \sS_{+}: \quad \left(\frac{{\vu^{(i)}_{t}} }{|| \vu^{(i)}_{t} ||}\right) \cdot \vx \not \geq \beta$. This follows since otherwise $\sV^{+}_{\beta}(\sS)$ and $\sV^{-}_{\beta}(\sS)$ would not be empty in contradiction to assumption 4.

Next, under the network being in an NAR assumption we have for all $t > T_{Margin}$:
\begin{enumerate}
\label{Appendix:opposite NAR to PAR points conditions proof}
    \item $\forall \vx \in \sS_{-}: \quad \left(\frac{{\vw^{(i)}_{t}} }{|| \vw^{(i)}_{t} ||}\right) \cdot \vx \leq -\beta$
    \item $\forall \vx \in \sS_{+}: \quad \left(\frac{{\vu^{(i)}_{t}} }{|| \vu^{(i)}_{t} ||}\right) \cdot \vx \leq -\beta$
\end{enumerate}
Thus, for all $t>T_{Margin}$, the network is in PAR$(\beta)$.

\subsection{PAR alignment direction}
\label{Appendix:proof of PCR alignment}



Now we will find where the parameters converge to when the network is in the PAR($\beta$). By \thmref{thm:Neural Alignment}, the normalized gradient flow converges to a perfectly clustered solution, i.e., $\underset{t \rightarrow \infty}{\lim} \hat{\mW}_{t} \coloneqq \hat{\mW}_{*}$ is of a perfectly clustered form. Formally that means $\exists \beta$ and $\exists \delta$ such that the normalized parameters $\hat{\mW}$ are of the form $\hat{\mW}_{*} = (\beta \tilde{\vw},\dots,\beta \tilde{\vw},\delta \tilde{\vu},\dots,\delta \tilde{\vu} ) \in \R^{2kd}$ and WLOG we can assume $||\tilde{\vw}|| = ||\tilde{\vu}|| = 1$. 

Because the solution is in the PAR($\beta$), the network margins are given as follows for positive points:
\begin{align*}
\forall \vx_{i} \in \sS_{+}: q_{i}(\mW) = y_{i}N_{\mW}(\vx_{i}) =y_{i}||\mW|| N_{\hat{\mW}}(\vx_{i}) &= v||\mW||\left(\sum \limits_{i=1}^{k} \sigma(\beta \tilde{\vw} \cdot \vx_{i}) - \sigma(\delta \tilde{\vu} \cdot \vx_{i})\right) \\ &= v||\mW||\left(k \beta \tilde{\vw} \cdot \vx_{i} -\alpha k \delta \tilde{\vu} \cdot \vx_{i}\right)    
\end{align*}

and negative points:
\begin{align*}
\forall \vx_{i} \in \sS_{-}: q_{i}(\mW) = y_{i}N_{\mW}(\vx_{i}) =y_{i}||\mW|| N_{\hat{\mW}}(\vx_{i}) &= v||\mW|| \left(\sum \limits_{i=1}^{k} \sigma(\delta \tilde{\vu} \cdot \vx_{i}) - \sigma(\beta \tilde{\vw} \cdot \vx_{i})\right) \\ &= v||\mW||\left(k\delta \tilde{\vu} \cdot \vx_{i} - \alpha k\beta \tilde{\vw} \cdot \vx_{i}\right)  
\end{align*}

where we used the fact we know the normalized solution would has a perfectly clustered form. We denote $\tilde{\beta} \coloneqq ||\mW|| \cdot \beta$ and similarly $\tilde{\delta} \coloneqq ||\mW|| \cdot \delta$

Using the above notations, the max margin problem in  \citet{lyu2019gradient} (Theorem A.8. therein) takes the form:
\begin{gather*}
\label{eq:kkt problem}
    \underset{\tilde{\beta} \in \R, \tilde{\delta} \in \R}{\argmin}\quad  k\tilde{\beta}^{2} + k\tilde{\delta}^2 = \underset{\tilde{\beta} \in \R, \tilde{\delta} \in \R}{\argmin} \quad  v^{2}k^{2}\tilde{\beta}^{2} + v^{2}k^{2}\tilde{\delta}^2 
    \\
    \forall \vx_{+} \in \sS_{+} : vk\tilde{\beta}\tilde{\vw} \cdot \vx_{+} - \alpha v k\tilde{\delta} \tilde{\vu} \cdot \vx_{+} \geq 1
    \\
    \forall \vx_{-} \in \sS_{-} : vk\tilde{\delta} \tilde{\vu} \cdot \vx_{-} - \alpha v k \tilde{\beta} \tilde{\vw} \cdot \vx_{-} \geq 1
\end{gather*}

Now we can denote $\vw \coloneqq v k\tilde{\beta} \tilde{\vw}$ and $\vu \coloneqq v k \tilde{\delta} \tilde{\vu}$ and reach the desired formulation:
\begin{gather*}
    \displaystyle
    \underset{\vw \in \R^{d}, \vu \in \R^{d}}{arg\min} \quad ||\vw ||^{2} + || \vu ||^{2}
    \\
    \forall \vx_{+} \in \sN_{+} : {\vw} \cdot \vx_{+} - \alpha{\vu} \cdot \vx_{+} \geq 1
    \\
    \forall \vx_{-} \in \sN_{-} : {\vu} \cdot \vx_{-} - \alpha {\vw} \cdot \vx_{-} \geq 1
\end{gather*}
We obtained a reformulation of \hyperref[eq:kkt problem]{(P)} as an SVM problem with variables $(\vw,\vu) \in \R^{2d}$ and with a transformed dataset which is a concatenated version of the original data $\phi(\vx) = [\sigma'(\vw^{*} \cdot \vx)\vx,-\sigma'(-\vw^{*} \cdot \vx)\vx]\in \R^{2d}$, where for $\vx_{+} \in \sN_{+}$, $\phi(\vx_{+}) = (\vx_{+},-\alpha\vx_{+})\in \R^{2d}$
and for $\vx_{-} \in \sN_{-}$,   $\phi(\vx_{-})=(-\alpha\vx_{-}, \vx_{-})\in \R^{2d}$.
\begin{figure}[t!]
    \centering
        \includegraphics[scale=0.5]{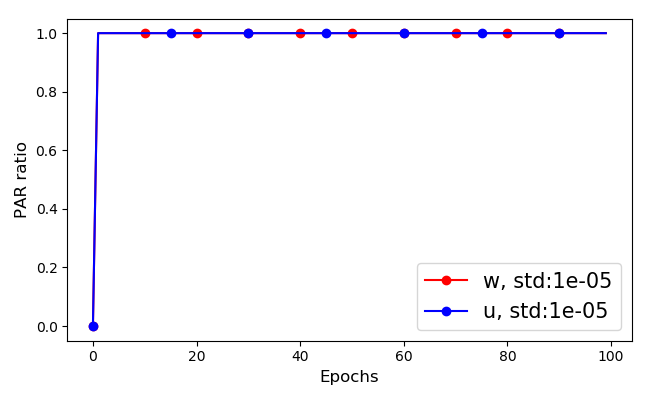}
    \caption[ ]{The ratio of neurons from each type in the PAR throughout the training process. We sample 400 data points from two antipodal separable Gaussians (one for each label) in $\R^{50}$. Our network is of 100 neurons (50 of each type) optimized on the data using SGD with batch size 1 with learning rate $\eta = 10^{-3}$.}
    \label{fig:Neurons in the PAR ratio}
\end{figure}

\section{Proof of Lemma \ref{lem:symmetric}}

Assume $\sV^{+}_{\beta}(\sS) \not = \emptyset$, i.e. $\exists \vv \in \sS $, s.t. $\forall \vx \in \sS_{+} \hat{\vv}\cdot \vx \geq \beta$ and $\exists \vx_{*} \in \sS_{-}$ s.t. $\hat{\vv}\cdot \vx_{*} \geq \beta$. This means that $\hat{\vv}\cdot -\vx_{*} \leq -\beta$, because the data is linearly separable $-\vx_{*} \in \sS$ has to be a positive point and by the definition of $\sV^{+}_{\beta}(\sS)$ that would mean $\hat{\vv}\cdot -\vx_{*} \geq \beta$ in contradiction.

By symmetry, if we assume $\sV^{-}_{\beta}(\sS) \not = \emptyset$ by taking the positive point which $\hat{v} \in \sV^{-}_{\beta}(\sS)$ mistakenly classifies as a negative one, we'll reach a contradiction again.

Therefore if $\forall \vx \in \sS,-\vx \in \sS$ we have $\sV_{\beta}^{+}(\sS) = \emptyset$ and $\sV_{\beta}^{-}(\sS) = \emptyset$ and Assumption 4 in \thmref{thm:multi neuron PAR}. holds in this case.


    
    

\newpage
\section{Entrance to PAR - High Dimensional Gaussians}
We will show that the entrance to the PAR indeed happens empirically for two separable Gaussians. We measure the percentage of neurons which are in the PAR of both types. A $\vw$ type neuron is considered in the PAR if it classifies like the ground truth $\vw^{*}$. A $\vu$ type neuron is considered in the PAR if it classifies like $-\vw^{*}$.

\begin{figure*}[t!]
    \centering
    \begin{subfigure}[b]{0.475\textwidth}  
        \centering 
        \includegraphics[scale=0.5]{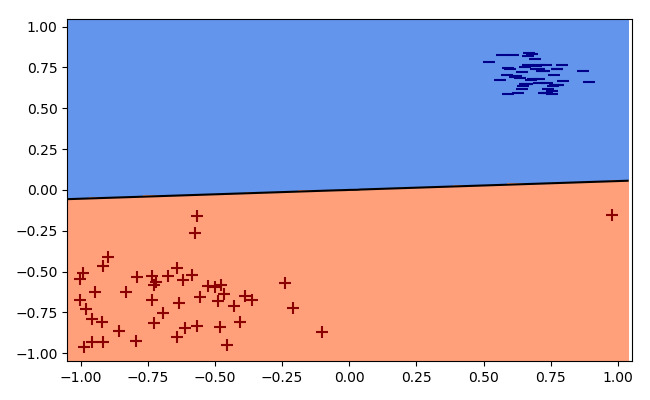}
        \caption[]%
        {$\sgn{N_{\mW}(\vx)}$}    
        \label{fig:NAR entrance - prediction landscape}
    \end{subfigure}
    \begin{subfigure}[b]{0.475\textwidth}  
        \centering 
        \includegraphics[scale=0.5]{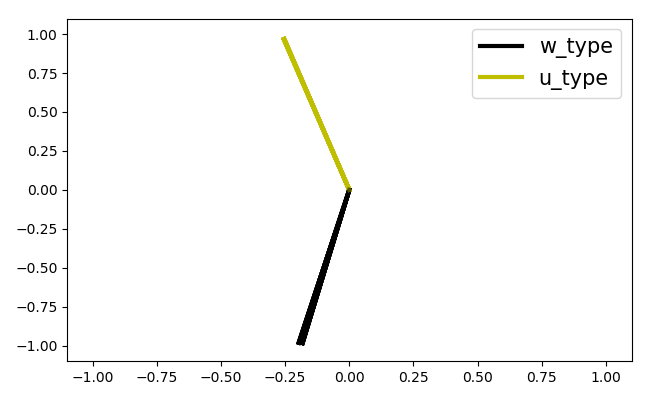}
        \caption[]%
        {Neurons Directions}    
        \label{fig:NAR entrance - neurons}
    \end{subfigure}
    \vspace{5 mm}
    \begin{subfigure}[b]{0.475\textwidth}  
        \centering 
        \includegraphics[scale=0.5]{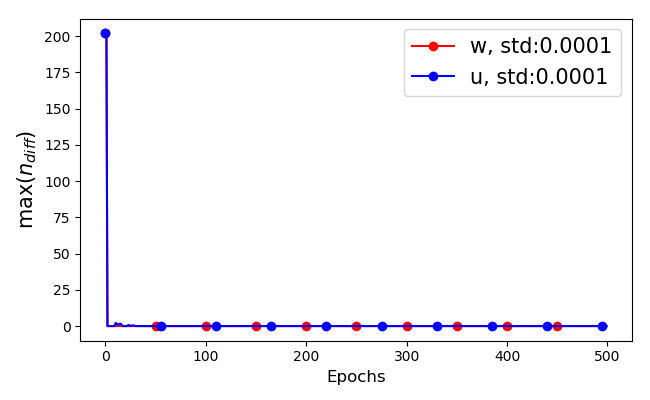}
        \caption[]%
        {Entrance to NAR}    
        \label{fig:NAR entrance - max points distance}
    \end{subfigure}
    \begin{subfigure}[b]{0.475\textwidth}
        \centering
        \includegraphics[scale=0.5]{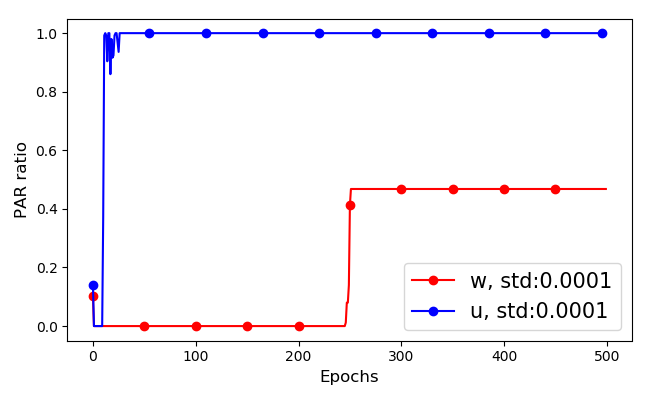}
        \caption[]%
        {Non Entrance to PAR}    
        \label{fig:NAR entrance - PAR ratio}
    \end{subfigure}
    \caption[ ]{
    The entrance to the NAR of a 100 neurons network.The weights initialization std is $10^{-4}$, learning rate is $\eta = 10^{-2}$. Each line in (c) and (d) is averaged over 5 initializations. }
    \label{fig:NAR entrance}
\end{figure*}

The percentage of neurons in the PAR throughout the training process is given in \figref{fig:Neurons in the PAR ratio}. We can see that the network enters the PAR.

\section{Entrance to NAR which is not a PAR}
In this section we show that learning can enter an NAR which is not a PAR. We sample two antipodal Gaussians and add one outlier positive point. Then for each neuron type ($\vw$ or $\vu$) we measure the maximum amount of data points classification disagreements between neurons of the same type denoted $\max(n_{diff})$ and the percentage of neurons which are in the PAR.

In \figref{fig:NAR entrance - prediction landscape} we can see that the network yields $100\%$ prediction accuracy. In \figref{fig:NAR entrance - neurons} we can see the directions of the neurons ($\vw$ type in black and $\vu$ type in yellow). In \figref{fig:NAR entrance - max points distance} we can see that the maximal number of points which neurons of the same type classified differently goes to zero, therefore all neurons of the same type agree on the classification of the data points. In \figref{fig:NAR entrance - PAR ratio} we can see that the ratio of $\vw$ type neurons which perfectly classifies the data does not increase to $1$ so the network does not enter the PAR. 
\section{Extension - First Layer Bias Term}
\label{first layer bias term}
In order to extend our results to include a bias term in the first layer, we would just need to reformulate our data points $\sS$ to $\sS'$ by $$(\vx ,y) \in \sS \subseteq \R^{d} \times \sY \mapsto ((\vx,1),y) \in \sS' \subseteq \R^{d+1} \times \sY$$
and extend our neurons to include a bias term: $$\forall \quad 1\leq i \leq k \quad \vw_{t}^{(i)} \in \R^{d} \mapsto (\vw_{t}^{(i)},b_{w}^{(i)}) \in \R^{d+1}, \quad \vu_{t}^{(i)} \in \R^{d} \mapsto (\vu_{t}^{(i)},b_{u}^{(i)}) \in \R^{d+1}$$
This is equivalent to reformulating the first weights matrix $\mW \in \R^{2k \times d} \mapsto \mW' \in \R^{2k \times (d+1)}$.

This reformulation is equivalent to adding a bias term for every neuron in the first layer, and all of the following results would still hold under the above reformulation.

The proofs of \thmref{thm:Risk Convergence} and \thmref{thm:difference between clustered leaky ReLus} follow exactly if we exchange $\mW$ with $\mW'$ while for the proofs of \thmref{thm:Neural Alignment} and \thmref{thm:multi neuron PAR} we use results from \cite{lyu2019gradient} and \cite{ji2020directional} that require the model to be homogeneous. Note that if we add a bias in the \textit{first} layer, the model remains homogeneous and the proofs of \thmref{thm:Neural Alignment} and \thmref{thm:multi neuron PAR} still hold for those cases as well.




\end{document}